\documentclass[journal]{IEEEtran}

\usepackage{graphicx}
\usepackage{amsmath,amssymb}
\usepackage{amsthm}

\newtheorem{lemma}{Lemma}[section]
\newtheorem{proposition}{Proposition}[section]
\newtheorem{theorem}{Theorem}[section]

\usepackage{xcolor}
\newcommand{\blue}[1]{\textcolor{black}{#1}}
\newcommand{\red}[1]{\textcolor{black}{#1}}

\graphicspath{{figures/}{Fig/}}
\usepackage{algorithm,algorithmic}
\usepackage{hyperref}
\usepackage{multirow}

\title{Cloud-Edge Collaborative Large Models for Robust Photovoltaic Power Forecasting}

\author{Nan QIAO, Shuning WANG, Sijing DUAN, Wenpeng CUI, Yuzhe CHEN, Qingchen YANG, Xingyuan HUA, and Ju REN%
\thanks{Nan QIAO and Shuning WANG contributed equally to this work.}%
\thanks{Nan QIAO and Shuning WANG are with the School of Computer Science and Engineering, Central South University, Changsha, Hunan 410083, China, and also with the Department of Computer Science and Technology, BNRist, Tsinghua University, Beijing 100084, China.}%
\thanks{Sijing DUAN, Xingyuan HUA, and Ju REN are with the Department of Computer Science and Technology, BNRist, Tsinghua University, Beijing 100084, China.}%
\thanks{Wenpeng CUI, Yuzhe CHEN, and Qingchen YANG are with Beijing Smartchip Microelectronics Technology Co., Ltd, Beijing 102200, China.}%
\thanks{Corresponding author: Sijing DUAN (duansj@tsinghua.edu.cn).}}
\begin{document}

\maketitle
\begin{abstract}
% Photovoltaic (PV) power forecasting in edge-enabled grids requires balancing forecasting accuracy, robustness under weather-driven distribution shifts, and strict latency constraints. Local specialized models are efficient for routine conditions but often degrade under rare ramp events and unseen weather patterns, whereas always relying on cloud-side large models incurs substantial communication delay and cloud overhead. To address this challenge, we propose a risk-aware cloud-edge collaborative framework for latency-sensitive PV forecasting. The framework integrates a site-specific expert predictor for routine cases, a lightweight edge-side model for enhanced local inference, and a cloud-side large retrieval model that provides matched historical context when needed through a retrieval-prediction pipeline. A lightweight screening module estimates predictive uncertainty, out-of-distribution risk, weather mutation intensity, and model disagreement, while a Lyapunov-guided router selectively escalates inference to the edge-small or cloud-assisted branches under long-term latency, communication, and cloud-usage constraints. The outputs of the activated branches are combined through adaptive fusion. Experiments on two real-world PV datasets demonstrate a favorable overall trade-off among forecasting accuracy, routing quality, robustness, and system efficiency.
Photovoltaic (PV) power forecasting in edge-enabled grids requires balancing forecasting accuracy, robustness under weather-driven distribution shifts, and strict latency constraints. Existing models work well under normal conditions but often struggle with rare ramp events and unexpected weather changes. Relying solely on cloud-based large models often leads to significant communication delays, which can hinder timely and efficient forecasting in practical grid environments. To address these issues, we propose a condition-adaptive cloud-edge collaborative framework \emph{CAPE} for PV forecasting. \emph{CAPE} consists of three main modules: a site-specific expert model for routine predictions, a lightweight edge-side model for enhanced local inference, and a cloud-based large retrieval model that provides relevant historical cases when needed. These modules are coordinated by a screening module that evaluates uncertainty, out-of-distribution risk, weather mutations, and model disagreement. Furthermore, we employ a Lyapunov-guided routing strategy to dynamically determine when to escalate inference to more powerful models under long-term system constraints. The final forecast is produced through adaptive fusion of the selected model outputs. Experiments on two real-world PV datasets demonstrate that \emph{CAPE} achieves superior performance in terms of forecasting accuracy, robustness, routing quality, and system efficiency.
\end{abstract}


\section{Introduction}\label{sec:intro}
With the rapid growth of distributed photovoltaic (PV) generation, short-term forecasting is critical for secure dispatch and frequency regulation in modern power systems \cite{Hong2016,Antonanzas2016,Luo2021PCLSTM}. Unlike conventional load forecasting, PV prediction is dominated by highly non-stationary meteorological factors, such as cloud advection and abrupt ramp events, causing accuracy to deteriorate sharply under rapidly changing weather conditions \cite{Si2021,Zang2024}. This challenge is amplified by the shift toward edge intelligence, where smart meters and edge controllers require low-latency decision support across heterogeneous sites and equipment configurations \cite{Gooi2023,Li2022EdgeCloud,Li2020Edgent,Qiao2026FOVA}. Furthermore, many newly deployed PV plants operate in data-scarce regimes, relying on transferable cross-site knowledge rather than extensive site-specific retraining \cite{Sarmas2022DataScarcity,Paletta2024CrossSite,Ansari2024Chronos,Das2024TimesFM}. 
% Because PV forecasting increasingly depends on heterogeneous observations (e.g., historical power, meteorological covariates, and sky images) that specialized models struggle to exploit jointly \cite{Si2021,Zang2024,Lin2025PVVLM}, there is a strong motivation to adopt large models with robust zero-shot transfer capabilities and unified multi-source reasoning \cite{wu2024netllm,Ansari2024Chronos,Das2024TimesFM,Woo2024Moirai,Hollmann2025TabPFN,Zhang2025LimiX}.
Since PV forecasting increasingly depends on heterogeneous observations like historical power, meteorological covariates, and sky images, specialized models often struggle to exploit these diverse data sources jointly \cite{Si2021,Zang2024,Lin2025PVVLM}. This challenge creates a strong motivation to adopt large models known for their robust zero-shot transfer capabilities and unified multi-source reasoning \cite{wu2024netllm,Ansari2024Chronos,Das2024TimesFM,Woo2024Moirai,Hollmann2025TabPFN,Zhang2025LimiX}.

However, the need for large models in PV forecasting extends beyond mere parameter scale. Most existing predictors are correlation-driven, performing reliably only within the operating regimes covered by their training data \cite{liu2024itransformer,wang2024timemixer,Das2024TimesFM}. Under unseen cloud evolution or extreme weather, they often rely on spurious correlations and fail precisely during the most challenging intervals for grid operation \cite{Luo2021PCLSTM,Luo2026PMWC}. Therefore, large models capable of capturing stable, physically meaningful dependencies are particularly appealing. By emphasizing robust structural relationships among irradiance, weather dynamics, and PV output, these models can improve generalization under distribution shifts and provide the interpretability needed for complex dispatch decisions \cite{Luo2026PMWC,Ma2026CausalFM,duan2025exploring}. Beyond passive prediction, extracting informative context from cross-site historical cases supports more reliable reasoning under evolving weather patterns. Recent studies confirm that large-model priors enhance transferability in intra-hour PV forecasting \cite{TimeLLM2024,Lin2025TimeLLMPV,Lin2025PVVLM}, while large models pretrained on heterogeneous structured data have shown strong potential for unified context extraction \cite{Zhang2025LimiX}.

Translating this potential into practice, however, putting them into practical PV forecasting systems faces a fundamental systems bottleneck: neither a pure edge nor a pure cloud architecture is viable.
On the one hand, relying exclusively on local edge predictors is increasingly inadequate for complex scenarios.
Specialized local models often lack generalization and tend to fail when confronted with rare events or unseen extreme weather.
Resource constraints on edge devices further limit their ability to support the computational demands of large models \cite{lin2025pushing}.
Many edge controllers possess limited computing and memory capacity, making it infeasible to run large-model inference and to store the comprehensive system-wide historical data needed for robust causal generalization \cite{Gooi2023,Li2022EdgeCloud,Li2020Edgent}.
On the other hand, a cloud-only pipeline is also problematic.
Continuously querying a large causality-aware model in the cloud leads to significant communication delays, bandwidth bottlenecks, and centralized resource contention, all of which are incompatible with the real-time demands of latency-sensitive edge operations \cite{Ren2023Survey,Jin2024CECoLLM}.
Ultimately, an inherent tension exists between edge-side resource limitations and cloud-side communication latency in PV forecasting.
% Therefore, how to navigate the dilemma between edge-side resource limitations and cloud-side communication latency to achieve efficient and accurate prediction remains a significant challenge.

To address this tension, we propose a two-stage cloud-edge collaborative prediction framework: cloud-side retrieval coupled with edge-side prediction.
Specifically, a large-scale data model in the cloud utilizes grid-wide historical records to perform correlation-aware retrieval, selecting a compact yet highly relevant subset of past cases. This distilled context is then transmitted to the edge, where a lightweight model conducts the final local inference. When the retrieved data is well matched to the current context, the forecasting accuracy relies more on the effectiveness of the cloud-side retrieval than on the complexity of the local model. Routine operating conditions remain manageable with purely local inference, while only a small fraction of scenarios require cloud-assisted large-model inference, such as ramp events, novel weather patterns, or high-uncertainty intervals. Nevertheless, this cloud-edge collaborative forecasting architecture faces several challenges, including determining the optimal timing for cloud retrieval, managing scheduling congestion, and adaptively fusing outputs from cloud and edge predictors to meet real-time demands.

% To this end, we propose a \underline{c}ondition-adaptive cloud-edge collaborative forecasting framework \emph{CAPE}, for latency-sensitive PV power prediction.
To overcome these challenges, we develop \emph{CAPE}, a \underline{c}ondition-\underline{a}da\underline{p}tive cloud-\underline{e}dge collaborative forecasting framework designed for latency-sensitive PV power prediction.
Specifically, a cloud-side large model performs cross-site case retrieval and context extraction from multi-source data.
An edge-side small model combines local measurements with cloud-provided matched data to generate robust forecasts, while a site-specific expert model serves as the default fast branch.
A screening module evaluates predictive uncertainty, distribution shift, weather mutation intensity, and model disagreement, and activates the cloud retrieval branch only when necessary.
The final forecast is obtained through confidence-aware fusion of the expert, edge-small-model, and cloud-assisted branches.
In this way, the framework preserves the speed advantage of expert edge models while leveraging the robustness and transferability of a large model in challenging scenarios.
In a nutshell, our main contributions are summarized as follows:
% Our main contributions are summarized as follows:
\begin{itemize}
% \item We formulate latency-sensitive PV forecasting as a condition-adaptive cloud-edge collaborative inference problem and propose a three-branch architecture. This framework comprises a cloud-side large model, an edge-side small model, and a site-specific expert predictor via selective invocation and confidence-aware fusion to jointly address accuracy, latency, communication, and cloud-usage considerations.
\item We formulate latency-sensitive PV forecasting as a condition-adaptive cloud-edge collaborative inference problem and propose a three-branch architecture. This framework integrates a cloud-side large model, an edge-side small model, and a site-specific expert predictor via selective invocation and confidence-aware fusion to jointly address accuracy, latency, communication, and cloud-usage considerations.
\item We prove that the optimal collaboration policy admits a routing-score threshold structure featuring a mean-field cloud-load equilibrium, and we establish a Lyapunov bound demonstrating an $\mathcal{O}(1/V)$ optimality gap and an $\mathcal{O}(V)$ queue backlog under long-term latency, communication, and cloud-usage constraints.
\item We conduct extensive experiments on two real-world PV datasets, demonstrating the superior performance of the proposed framework.
% \end{itemize}
\end{itemize} 
\section{Related Work}
\label{sec:related_work}

\paragraph{PV Power Forecasting Methods.}
PV power forecasting has been extensively studied from statistical, physical, and data-driven perspectives, and prior studies have shown that forecasting performance is highly dependent on weather variability, forecast horizon, and information availability \cite{Antonanzas2016,Hong2016}. Recent deep-learning-based methods have improved forecasting accuracy by explicitly exploiting temporal dependence and domain knowledge. For example, physics-constrained recurrent models incorporate PV-specific operational priors to alleviate unreasonable predictions under standard purely data-driven learning \cite{Luo2021PCLSTM}. To better handle rapid irradiance changes caused by cloud movement, recent studies further integrate satellite imagery or sky images and demonstrate that visual cloud information is particularly valuable in ultra-short-term and intra-hour PV forecasting \cite{Si2021,Zang2024,Lin2025PVVLM}. 

More recently, large-model-based PV forecasting has emerged, including Time-LLM-style sequence reprogramming and multimodal vision-language prediction, suggesting that large-model priors can improve transferability and cross-site generalization \cite{TimeLLM2024,Lin2025TimeLLMPV,Lin2025PVVLM}. Similarly, large weather models have also been introduced into decentralized PV forecasting through spatio-temporal knowledge distillation, further indicating that pretrained large models can enhance site-level prediction in distributed deployment settings \cite{he2025stkdpv}. However, most existing PV forecasting studies still focus on maximizing predictive accuracy for a single inference path, while the practical issues of selective large-model invocation, communication overhead, and cloud-edge deployment remain insufficiently explored \cite{Lin2025TimeLLMPV,Lin2025PVVLM,he2025stkdpv,Li2022EdgeCloud}.

\paragraph{Large Models for Time-Series Forecasting.}
Large models for time series have advanced rapidly. Recent architectures like iTransformer and TimeMixer have improved multivariate forecasting through inverted tokenization and decomposable multiscale mixing \cite{liu2024itransformer,wang2024timemixer}. Moreover, pretrained foundation models, such as Chronos, TimesFM, Moirai, and MOMENT, demonstrate strong zero-shot and few-shot capabilities across diverse benchmarks, highlighting the benefits of large-scale pretraining for transferability \cite{Ansari2024Chronos,Das2024TimesFM,Woo2024Moirai,goswami2024moment}. Concurrently, Time-LLM bridges general-purpose language priors and temporal prediction by reprogramming language models for time-series tasks \cite{TimeLLM2024}. For heterogeneous inputs, large structured-data models like TabPFN and LimiX offer unified in-context prediction and strong generalization beyond conventional per-dataset training pipelines \cite{Hollmann2025TabPFN,Zhang2025LimiX}. Specifically, LimiX supports retrieval-style conditional inference, making it highly suitable for grid forecasting tasks dominated by diverse structured observations \cite{Zhang2025LimiX}. While incorporating physically meaningful dependencies or structural reasoning has been shown to improve predictive robustness under environmental shifts \cite{Luo2021PCLSTM,Luo2026PMWC,Ma2026CausalFM}, the practical application of these large structured-data models to PV forecasting remains largely unexplored, particularly under latency-sensitive cloud-edge deployment constraints \cite{Zhang2025LimiX,Ma2026CausalFM,Li2022EdgeCloud}.

\paragraph{Cloud-Edge Collaborative Inference.}
Edge intelligence has become an important paradigm for smart grids because many sensing, forecasting, and control tasks demand low-latency local processing and cannot rely solely on centralized cloud execution \cite{Gooi2023,Li2022EdgeCloud,Qiao2026FORLER,huang2024multi,akram2025comprehensive}. Existing edge-AI studies show that collaborative inference between cloud and edge can effectively reduce end-to-end latency through model partitioning, adaptive execution, and early exit, thereby outperforming naive cloud-only deployments in resource-constrained settings \cite{Li2020Edgent,Ren2023Survey}. For large models, CE-CoLLM further demonstrates that direct cloud-edge collaboration is often bottlenecked by contextual-information transmission, and that adaptive switching between standalone edge inference and cloud-assisted inference is necessary to balance efficiency and accuracy \cite{Jin2024CECoLLM}. Meanwhile, even PV-specific distributed forecasting studies, such as decentralized forecasting enhanced by large weather models via spatio-temporal knowledge distillation, mainly improve predictive quality rather than addressing online selective invocation and system-level latency--communication trade-offs \cite{he2025stkdpv}. 
% Beyond inference acceleration itself, related studies on distributed decision-making and system optimization also highlight the importance of communication-aware coordination in networked intelligent systems: PoPeC investigates PAoI-centric task offloading with priority over unreliable channels, while FOVA and FORLER study offline federated reinforcement learning under mixed-quality data and actor-rectified Q-ensemble designs, respectively \cite{Qiao2026FOVA,Qiao2026FORLER}. 

These works provide useful insights into delay-aware offloading and robust distributed policy learning, but they do not address the forecasting-specific problem of selectively invoking large predictive models under sample-varying uncertainty and weather-induced non-stationarity. Therefore, a forecasting-oriented collaborative framework must jointly address condition screening, selective cloud invocation, communication cost, and confidence-aware fusion, which is not sufficiently handled by existing literature \cite{Li2022EdgeCloud,Ren2023Survey,Jin2024CECoLLM,Zang2024}.

% =========================================
% System Model and Problem Formulation
% =========================================
\section{System Model and Problem Formulation}
\label{sec:system_model}

In this section, we establish the system model and formulate an optimization problem for the cloud-edge PV forecasting framework. We first define the observations and inference modes in the multi-site cloud-edge system. We then model cloud delay, communication cost, and mode-dependent resource consumption, formulate the long-term stochastic optimization problem, and finally introduce virtual queues that transform the long-term constraints into an online-control form for the Lyapunov-guided method. As shown in Fig.~\ref{fig:system_model}, the system includes multiple PV sites with edge nodes, a centralized cloud, and wireless uplink/downlink links for query transmission and context feedback. Table \ref{tab:nomenclature} shows the main symbols in this paper.

\begin{figure}[htbp]
\centering
\includegraphics[width=0.8\linewidth]{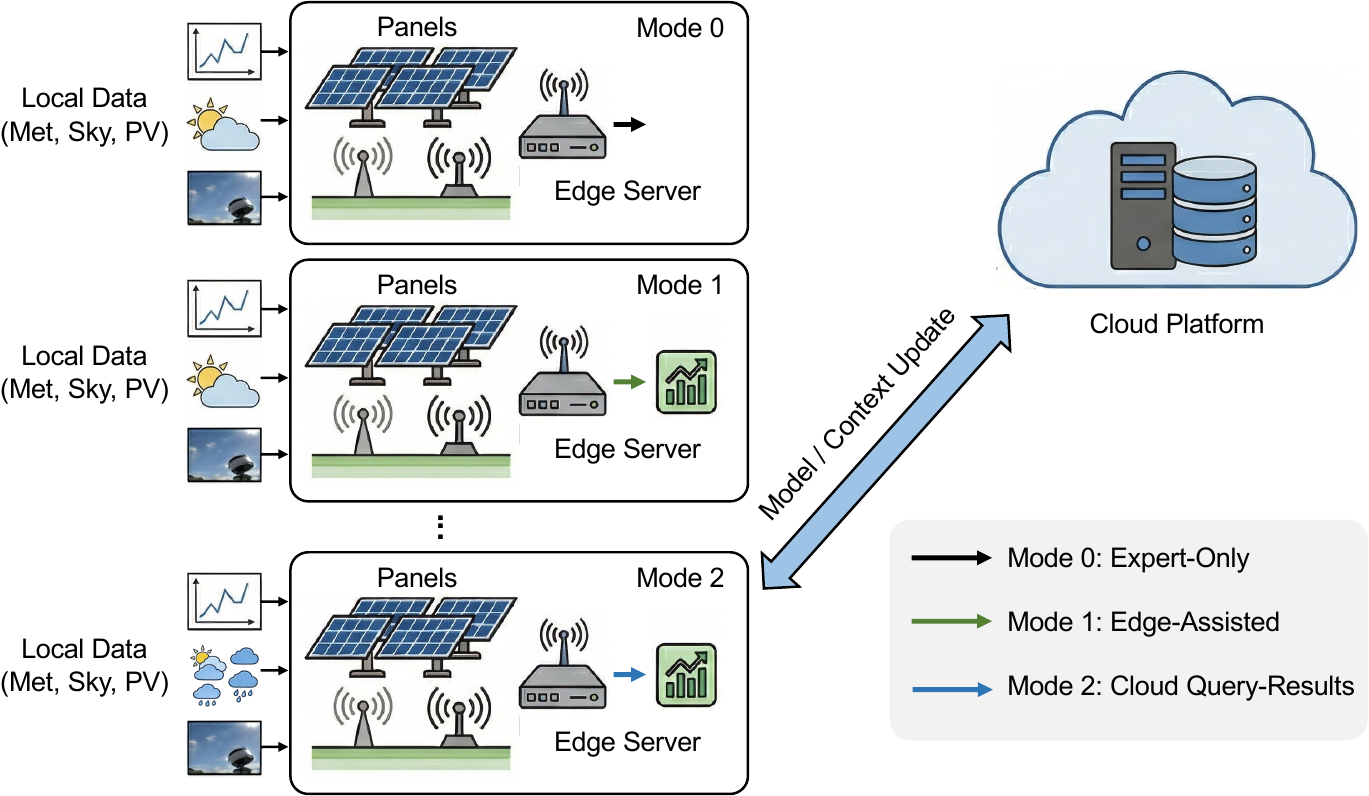}
\caption{System model of the proposed cloud-edge PV forecasting.}
\label{fig:system_model}
\vspace{-10pt}
\end{figure}
% Appendix \ref{app:notation}.

\begin{table}[htbp]
    \centering
    \footnotesize
    \caption{Summary of key notations.}
    \label{tab:nomenclature}
    \begin{tabular}{|c|l|}
        \hline
        \textbf{Symbol} & \textbf{Description} \\
        \hline
        $\mathcal{N}$, $N$ & Set of PV sites and the total number of sites \\
        \hline
        $t$, $i$ & Time slot index and edge node index \\
        \hline
        $\mathbf{s}_{i,t}$ & Raw local state of node $i$ at time $t$ \\
        \hline
        $\mathbf{x}_{i,t}^{\mathrm{hist}}$ & Local historical observation window \\
        \hline
        $\mathbf{m}_{i,t}$, $\mathbf{v}_{i,t}$ & Meteorological variables \\
        \hline
        $\mathbf{X}_{i,t}^{\mathrm{loc}}$ & Unified local model input \\
        \hline
        $\mathbf{y}_{i,t}^{*}$ & Future $H$-step PV output vector  \\
        \hline
        $\mathcal{Y}$, $H$ &  Bounded prediction domain and  horizon\\
        \hline
        $\mathcal{A}$, $a_{i,t}$ & Action space and the selected inference mode  \\
        \hline
        $\hat{\mathbf{y}}_{i,t}^{e}, \hat{\mathbf{y}}_{i,t}^{s}, \hat{\mathbf{y}}_{i,t}^{c}$ & \begin{tabular}{@{}l@{}}Forecast candidates from expert, small, and \\ cloud branches\end{tabular} \\
        \hline
        $\mathbf{q}_{i,t}$ & Compact query formed by the edge node \\
        \hline
        $\mathcal{S}_{i,t}$ & Support set returned by the cloud retriever \\
        \hline
        $\mathbf{z}_{i,t}$ & Cloud context summarizing the support set \\
        \hline
        $\mathcal{M}_{a}$ & Active branch set for inference mode $a$ \\
        \hline
        $\mathbf{w}_{i,t}^{(a)}$ & Mode-specific fusion weight vector \\
        \hline
        $\ell_{i,t}^{(a)}$ & Action-wise realized forecasting loss \\
        \hline
        $\rho(t)$ & Instantaneous cloud-request ratio  \\
        \hline
        $\tau_{i}^{(a)}$, $c_{i}^{(a)}$ & \begin{tabular}{@{}l@{}}Mode-dependent processing latency and  \\communication cost\end{tabular} \\
        \hline
        $\bar{L}, \bar{\tau}, \bar{c}, \bar{\rho}$ & \begin{tabular}{@{}l@{}}Long-term average loss, latency, \\ communication, and cloud usage\end{tabular} \\
        \hline
        $Q_{\tau}(t), Q_{c}(t), Q_{\rho}(t)$ & \begin{tabular}{@{}l@{}}Virtual queues for latency, communication, \\ and cloud usage constraints\end{tabular} \\
        \hline
    \end{tabular}
    \vspace{-9pt}
\end{table}

\subsection{Multi-Site Cloud-Edge Forecasting System}
We consider a cloud-edge collaborative PV forecasting system with a set of PV sites
\(
\mathcal{N}=\{1,\ldots,N\}
\).
Time is slotted and indexed by \(t=0,1,2,\ldots\). At each slot, edge node \(i\in\mathcal{N}\) observes a raw local state
\begin{equation}
\mathbf{s}_{i,t}
=
\left(
\mathbf{x}_{i,t}^{\mathrm{hist}},
\mathbf{m}_{i,t},
\mathbf{v}_{i,t}
\right),
\end{equation}
where \(\mathbf{x}_{i,t}^{\mathrm{hist}}\) is the local historical observation window, \(\mathbf{m}_{i,t}\) denotes meteorological variables, and \(\mathbf{v}_{i,t}\) denotes optional visual information such as sky-image embeddings. These modalities are preprocessed into a unified local model input
\begin{equation}
\mathbf{X}_{i,t}^{\mathrm{loc}}
=
\Phi_i^{\mathrm{loc}}\left(\mathbf{s}_{i,t}\right),
\end{equation}
which is used by all downstream predictors 
\blue{ and retrieval modules}. Meteorological and visual cues are therefore incorporated whenever available, while the subsequent notation only needs the compact input \(\mathbf{X}_{i,t}^{\mathrm{loc}}\).
PV outputs are normalized by the site capacity. Accordingly,
\(
y_{i,t+h}\in[0,1]
\)
for each horizon step \(h\), and the bounded prediction domain is
\(
\mathcal{Y}\triangleq[0,1]^H
\).
The forecasting target is the future \(H\)-step PV output vector
\begin{equation}
\mathbf{y}_{i,t}^{*}
=
\left[
y_{i,t+1},\ldots,y_{i,t+H}
\right]^\top
\in\mathcal{Y}.
\end{equation}
The full target \(\mathbf{y}_{i,t}^{*}\) becomes observable only after slot \(t+H\). The realized forecasting loss for slot \(t\) is therefore revealed with delay.
At each time slot, the scheduler selects an inference mode
\(
a_{i,t} \in \mathcal{A} \triangleq \{0, 1, 2\},
\)
where \(a_{i,t}=0\), \(a_{i,t}=1\), and \(a_{i,t}=2\) denote the expert-only mode, the edge-only fusion mode, and the cloud-assisted fusion mode, respectively.
The expert candidate \(\hat{\mathbf{y}}_{i,t}^{e}\in\mathcal{Y}\) is always computed locally. The small-branch candidate \(\hat{\mathbf{y}}_{i,t}^{s}\in\mathcal{Y}\) is computed on demand only when \(a_{i,t}\in\{1,2\}\) \blue{, and the cloud-assisted candidate \(\hat{\mathbf{y}}_{i,t}^{c}\in\mathcal{Y}\) is computed on demand only when \(a_{i,t}=2\).}
 \blue{The cloud maintains a global historical case base \(\mathcal{D}^{\mathrm{cld}}\). When \(a_{i,t}=2\), the cloud-assisted branch is explicitly decoupled into a retrieval stage and a conditional prediction stage. Specifically, the edge first forms a compact query}
\begin{equation}
 \blue{
\mathbf{q}_{i,t}
=
h\left(\mathbf{X}_{i,t}^{\mathrm{loc}}\right),
}
\label{eq:query_sys}
\end{equation}
 \blue{the cloud retriever returns a support set}
\begin{equation}
 \blue{
\mathcal{S}_{i,t}
=
F^{\mathrm{cld}}\left(
\mathbf{q}_{i,t},
\mathcal{D}^{\mathrm{cld}},
K
\right),
}
\label{eq:retrieval_sys}
\end{equation}
 \blue{which is summarized into a cloud context}
\begin{equation}
 \blue{
\mathbf{z}_{i,t}
=
\psi\left(\mathcal{S}_{i,t}\right),
}
\label{eq:context_sys}
\end{equation}
 \blue{and the cloud-assisted candidate is then produced as}
\begin{equation}
 \blue{
\hat{\mathbf{y}}_{i,t}^{c}
=
f^c\left(
\mathbf{X}_{i,t}^{\mathrm{loc}},
\mathbf{z}_{i,t}
\right)
\in\mathcal{Y}.
}
\label{eq:cloud_candidate_sys}
\end{equation}
 \blue{\(F^{\mathrm{cld}}\) is instantiated as a cloud-scale retrieval model, whereas \(f^c\) is instantiated as a lightweight conditional regressor with much weaker dependence on model size and optional parameter sharing with the local small model. At the outer control level, this internal decoupling only refines the construction of branch \(c\); the action space, the active branch sets, and the fusion structure remain unchanged.}
 \blue{To formalize this interpretation, we view the retrieved support set as a proxy for an unobserved weather-site regime and state the following:}
\begin{proposition}
\label{prop:latent_decomp}
% [ \blue{Latent-Regime Decomposition of the Cloud-Assisted Branch}]
{Denote \(\xi_{i,t}\) as a latent weather-site regime variable. If}
% \begin{equation}
$\mathbf{y}_{i,t}^{*}
\perp\perp
\mathcal{S}_{i,t}
\mid
\bigl(
\mathbf{X}_{i,t}^{\mathrm{loc}},
\xi_{i,t}
\bigr),$
% \label{eq:latent_ci_sys}
% \end{equation}
{then the Bayes-optimal cloud-assisted predictor admits the decomposition}
$p(
\mathbf{y}_{i,t}^{*}
\mid
\mathbf{X}_{i,t}^{\mathrm{loc}},
\mathcal{S}_{i,t}
)
=
\int
p(
\mathbf{y}_{i,t}^{*}
\mid
\mathbf{X}_{i,t}^{\mathrm{loc}},
\xi
)
p\left(
\xi
\mid
\mathbf{X}_{i,t}^{\mathrm{loc}},
\mathcal{S}_{i,t}
\right)
d\xi.$
% \begin{equation}
% p\left(
% \mathbf{y}_{i,t}^{*}
% \mid
% \mathbf{X}_{i,t}^{\mathrm{loc}},
% \mathcal{S}_{i,t}
% \right)
% =
% \int
% p\left(
% \mathbf{y}_{i,t}^{*}
% \mid
% \mathbf{X}_{i,t}^{\mathrm{loc}},
% \xi
% \right)
% p\left(
% \xi
% \mid
% \mathbf{X}_{i,t}^{\mathrm{loc}},
% \mathcal{S}_{i,t}
% \right)
% \,d\xi.
% \nonumber
% \label{eq:latent_decomp_sys}
% \end{equation}
\end{proposition}
% \noindent
\blue{Thus, the cloud-assisted branch can be interpreted as a two-stage inference pipeline: retrieval infers a posterior regime summary from \(\mathcal{S}_{i,t}\), and the conditional predictor maps \((\mathbf{X}_{i,t}^{\mathrm{loc}},\mathbf{z}_{i,t})\) to the forecast. This provides a probabilistic justification for using a large retrieval model together with a lightweight prediction model.}
\blue{A short proof based on the law of total probability is deferred to Appendix~\ref{app:system_model_aux}.}
The corresponding active branch sets are
\begin{equation}
\mathcal{M}_{0}=\{e\},
\qquad
\mathcal{M}_{1}=\{e,s\},
\qquad
\mathcal{M}_{2}=\{e,s,c\}.
\end{equation}
% For each mode \(a\in\mathcal{A}\), 
The mode-specific fusion weight vector is denoted by
\(
\mathbf{w}_{i,t}^{(a)}\in\Delta_{|\mathcal{M}_{a}|},\forall a\in\mathcal{A},
\)
where \(\Delta_d\) is the \(d\)-dimensional probability simplex. Once action \(a_{i,t}\) is selected, the final forecast is
\begin{equation}
\hat{\mathbf{y}}_{i,t}
=
\sum_{m\in\mathcal{M}_{a_{i,t}}}
w_{i,m}^{(a_{i,t})}(t)\hat{\mathbf{y}}_{i,t}^{m},
\qquad
\mathbf{w}_{i,t}^{(a_{i,t})}\in\Delta_{|\mathcal{M}_{a_{i,t}}|}.
\label{eq:fused_output}
\end{equation}
For brevity, when the active mode is clear from the context, we write
\(
\mathbf{w}_{i,t}\equiv \mathbf{w}_{i,t}^{(a_{i,t})}.
\)
A multi-horizon forecasting loss is used. For any action \(a\in\mathcal{A}\), the action-wise realized loss is
\begin{equation}
\ell_{i,t}^{(a)}(\mathbf{w})
=
\ell
\left(
\mathbf{y}_{i,t}^{*},
\sum_{m\in\mathcal{M}_{a}}w_m\hat{\mathbf{y}}_{i,t}^{m}
\right),
\quad
\mathbf{w}\in\Delta_{|\mathcal{M}_{a}|},
\label{eq:inst_loss_actionwise}
\end{equation}
and the realized loss under the selected action is
\begin{equation}
\ell_{i,t}\big(\mathbf{w}_{i,t},a_{i,t}\big)
=
\ell_{i,t}^{(a_{i,t})}\big(\mathbf{w}_{i,t}^{(a_{i,t})}\big).
\label{eq:inst_loss}
\end{equation}
Typical choices of \(\ell(\cdot,\cdot)\) include MAE, weighted MAE, Huber loss, and squared loss on the bounded domain \(\mathcal{Y}\). These loss functions are standard on bounded prediction domains and admit finite Lipschitz constants on \(\mathcal{Y}\). Since the fused output in Eq.~\eqref{eq:fused_output} is affine in \(\mathbf{w}\), the action-wise fusion step in Eq.~\eqref{eq:inst_loss_actionwise} is a convex optimization over the simplex for the loss families considered here. A derivation is in Appendix~\ref{app:system_model_aux}.

\subsection{Mean-Field Cloud Delay and Resource-Consumption} % Model
 \blue{When edge nodes simultaneously request cloud-assisted retrieval, they share a common congestion-dependent cloud waiting component.} The instantaneous cloud-request ratio is
\begin{equation}
\rho(t)
=
\frac{1}{N}\sum_{i=1}^{N}\mathbb{I}\{a_{i,t}=2\}
\in[0,1].
\label{eq:rho_def}
\end{equation}
A mean-field delay model is adopted, in which the cloud waiting time is described by a deterministic congestion curve
\begin{equation}
d^{\mathrm{cld}}(t)=\phi\big(\rho(t)\big).
\label{eq:cloud_delay}
\end{equation}
A continuous nondecreasing choice of \(\phi:[0,1]\to\mathbb{R}_{+}\) captures the fact that heavier cloud contention leads to longer waiting time. This approximation is well motivated when \(N\) is large and the influence of an individual node on the aggregate cloud load is \(O(1/N)\). The same model can also be viewed as the continuum limit of a congestion game in which each edge node decides whether to request cloud service and incurs a congestion-dependent waiting cost.
The quantity \(\tau_{i}^{e}\) denotes the always-on expert-side latency, which includes the expert prediction itself together with the lightweight local screening and routing-score estimation computations required before routing. If the scheduler activates mode \(1\) or \(2\), the small branch is then executed on demand. If mode \(2\) is activated, the  \blue{cloud retrieval/communication} pipeline is launched in parallel with the on-demand small branch. The resulting mode-dependent latency of node \(i\) is
\begin{align}
\tau_{i}^{(0)}
&=
\tau_{i}^{e},\\
\tau_{i}^{(1)}
&=
\tau_{i}^{e}
+
\tau_{i}^{s}
+
\tau_{i}^{f},\\
% \label{eq:latency_mode2}
\tau_{i}^{(2)}\big(\rho(t)\big)
&=
\tau_{i}^{e}
+
\tau_{i}^{f}\\
+
\max&\left\{
\tau_{i}^{s},
\,
\tau_{i}^{\mathrm{up}}
+
\phi\big(\rho(t)\big)
+
\tau^{\mathrm{cld}}
+
\tau_{i}^{\mathrm{down}}
\right\},\nonumber
\end{align}
where  \blue{\(\tau_{i}^{s}\) denotes the additional small-branch latency, \(\tau_{i}^{f}\) denotes the local post-branch processing latency including fusion and, under mode \(2\), the lightweight conditional prediction after context download,} while  \blue{\(\tau_{i}^{\mathrm{up}}\), \(\tau^{\mathrm{cld}}\), and \(\tau_{i}^{\mathrm{down}}\) denote the query-construction/upload latency, cloud-side retrieval/context-construction latency, and context download latency, respectively.}
The communication cost is
$c_{i}^{(0)}=c_{i}^{(1)}=0,
c_{i}^{(2)}=\kappa_{i},$
where \(\kappa_i\) denotes the per-request uplink/downlink traffic volume of node \(i\) \blue{ for transmitting the compact query and the retrieved context}.
The normalized output range, together with finite computation and communication budgets, yields uniformly bounded one-slot quantities. Accordingly, finite constants \(L_{\mathrm{sup}},\tau_{\mathrm{sup}},c_{\mathrm{sup}}<\infty\) satisfy
\begin{equation}
0\le \ell_{i,t}^{(a)}(\mathbf{w})\le L_{\mathrm{sup}},
0\le \tau_i^{(a)}(\rho)\le \tau_{\mathrm{sup}},
0\le c_i^{(a)}\le c_{\mathrm{sup}},
\label{eq:bounded_slot}
\end{equation}
for all \(i\), \(t\), \(a\in\mathcal{A}\), \(\rho\in[0,1]\), and feasible \(\mathbf{w}\).

% For clarity, we summarize key notations in 

\subsection{Long-Term Stochastic Optimization Formulation}
The long-term average forecasting loss is defined as
\begin{equation}
\bar{L}
=
\limsup_{T\to\infty}
\frac{1}{T}
\sum_{t=0}^{T-1}
\mathbb{E}
\left[
\frac{1}{N}\sum_{i=1}^{N}
\ell_{i,t}\big(\mathbf{w}_{i,t},a_{i,t}\big)
\right].
\end{equation}
Similarly, the average latency, communication cost, and cloud-request ratio are defined as
\begin{equation}
\bar{\tau}
=
\limsup_{T\to\infty}
\frac{1}{T}
\sum_{t=0}^{T-1}
\mathbb{E}
\left[
\frac{1}{N}\sum_{i=1}^{N}
\tau_{i}^{(a_{i,t})}\big(\rho(t)\big)
\right],
\end{equation}
\begin{equation}
\bar{c}
=
\limsup_{T\to\infty}
\frac{1}{T}
\sum_{t=0}^{T-1}
\mathbb{E}
\left[
\frac{1}{N}\sum_{i=1}^{N}
c_{i}^{(a_{i,t})}
\right],
\end{equation}
\begin{equation}
\bar{\rho}
=
\limsup_{T\to\infty}
\frac{1}{T}
\sum_{t=0}^{T-1}
\mathbb{E}\big[\rho(t)\big].
\end{equation}
The objective is to minimize the long-term average forecasting loss subject to average system constraints:
\begin{subequations}
\label{prob:P1}
\begin{align}
\textbf{P1:}\qquad
\min_{\{a_{i,t},\mathbf{w}_{i,t}^{(a_{i,t})}\}}
\quad &
\bar{L}
\\
\text{s.t.}\quad
&
\bar{\tau}\le \tau_{\max},
\\
&
\bar{c}\le c_{\max},
\\
&
\bar{\rho}\le \rho_{\max},
\\
& 0<\rho_{\max}\le 1,
\\
&
a_{i,t}\in\{0,1,2\},
\\
&
\mathbf{w}_{i,t}^{(a_{i,t})}\in\Delta_{|\mathcal{M}_{a_{i,t}}|},\quad
\forall i,t.
\end{align}
\end{subequations}
Problem~\eqref{prob:P1} is challenging for 3 reasons: (i) the mode variable \(a_{i,t}\) is discrete; (ii) the cloud-assisted delay couples all edge nodes through the mean-field term \(\rho(t)\); and (iii) the current forecasting difficulty is stochastic and the true loss is observed only after delayed label revelation.
The optimization problem above is defined in terms of the realized losses \(\ell_{i,t}^{(a)}\). Since \(\mathbf{y}_{i,t}^{*}\) is not available at decision time, the online controller developed in Section~\ref{sec:method} will use calibrated conditional surrogate losses \(\widehat{\ell}_{i,t}^{(a)}\) to approximate these realized quantities.

\subsection{Virtual-Queue for Long-Term Constraints}
To transform \textbf{P1} into an online control problem, virtual queues are introduced for the average latency, communication, and cloud-usage constraints:
\begin{equation}
Q_{\tau}(t+1)
=
\left[
Q_{\tau}(t)
+
\frac{1}{N}\sum_{i=1}^{N}
\tau_{i}^{(a_{i,t})}\big(\rho(t)\big)
-
\tau_{\max}
\right]^{+}.
\label{eq:q_tau}
\end{equation}
\begin{equation}
Q_{c}(t+1)
=
\left[
Q_{c}(t)
+
\frac{1}{N}\sum_{i=1}^{N}
c_{i}^{(a_{i,t})}
-
c_{\max}
\right]^{+}.
\label{eq:q_c}
\end{equation}
\begin{equation}
Q_{\rho}(t+1)
=
\left[
Q_{\rho}(t)
+
\rho(t)
-
\rho_{\max}
\right]^{+}.
\label{eq:q_rho}
\end{equation}
The bounded one-slot quantities in Eq.~\eqref{eq:bounded_slot} ensure uniformly bounded increments of these virtual queues. Standard Lyapunov arguments further show that mean-rate stability of \(Q_{\tau}(t)\), \(Q_{c}(t)\), and \(Q_{\rho}(t)\),
\begin{equation}
\lim_{T\to\infty}
\frac{\mathbb{E}[Q_{g}(T)]}{T}=0,
\qquad
g\in\{\tau,c,\rho\},
\label{eq:mean_rate_stability}
\end{equation}
guarantees the corresponding time-average constraints in Eq.~\eqref{prob:P1}. The derivation is reported in Appendix~\ref{app:system_model_aux}.
The quadratic Lyapunov function is chosen as
\begin{equation}
\mathcal{L}\big(\mathbf{Q}(t)\big)
=
\frac{1}{2}
\left(
Q_{\tau}^{2}(t)+Q_{c}^{2}(t)+Q_{\rho}^{2}(t)
\right),
\end{equation}
where \(\mathbf{Q}(t)\triangleq(Q_{\tau}(t),Q_{c}(t),Q_{\rho}(t))\). The conditional drift is
\begin{equation}
\Delta(t)
=
\mathbb{E}
\left[
\mathcal{L}\big(\mathbf{Q}(t+1)\big)-\mathcal{L}\big(\mathbf{Q}(t)\big)
\mid
\mathbf{Q}(t)
\right].
\end{equation}

\section{Design of \emph{CAPE}}
\label{sec:method}

% \subsection{Overview}
% In this section, we propose a cloud-edge forecasting framework \emph{CAPE}. As illustrated in Fig.~\ref{fig:technical_framework}, a local expert serves as the default predictor, difficult samples are selectively escalated to the edge-small or cloud-assisted branch, and the final forecast is produced by fusion of the activated predictors with delayed online updates. Accordingly, we first introduce local screening and routing-score computation, then develop the Lyapunov-guided routing policy, next present the fusion and delayed-update mechanism, and finally establish the long-term performance guarantee.
In this section, we propose a cloud-edge forecasting framework \emph{CAPE}. As illustrated in Fig.~\ref{fig:technical_framework}, its workflow follows four sequential modules. First, a local screening mechanism evaluates inputs, enabling a local expert to yield default predictions while computing routing scores to identify hard samples. Then, a routing decision module dynamically offloads these challenging instances to edge or cloud predictors on demand, guided by Lyapunov optimization to handle mean-field cloud congestion. Next, a prediction and fusion module aggregates outputs from the selected predictors to generate the final forecasting result. Finally, an update and guarantee module leverages delayed labels to drive online model updates, which theoretically ensures long-term forecasting performance and queue stability across the system.

\begin{figure}[htbp]
\centering
\scalebox{1}[1]{\includegraphics[width=0.98\linewidth]{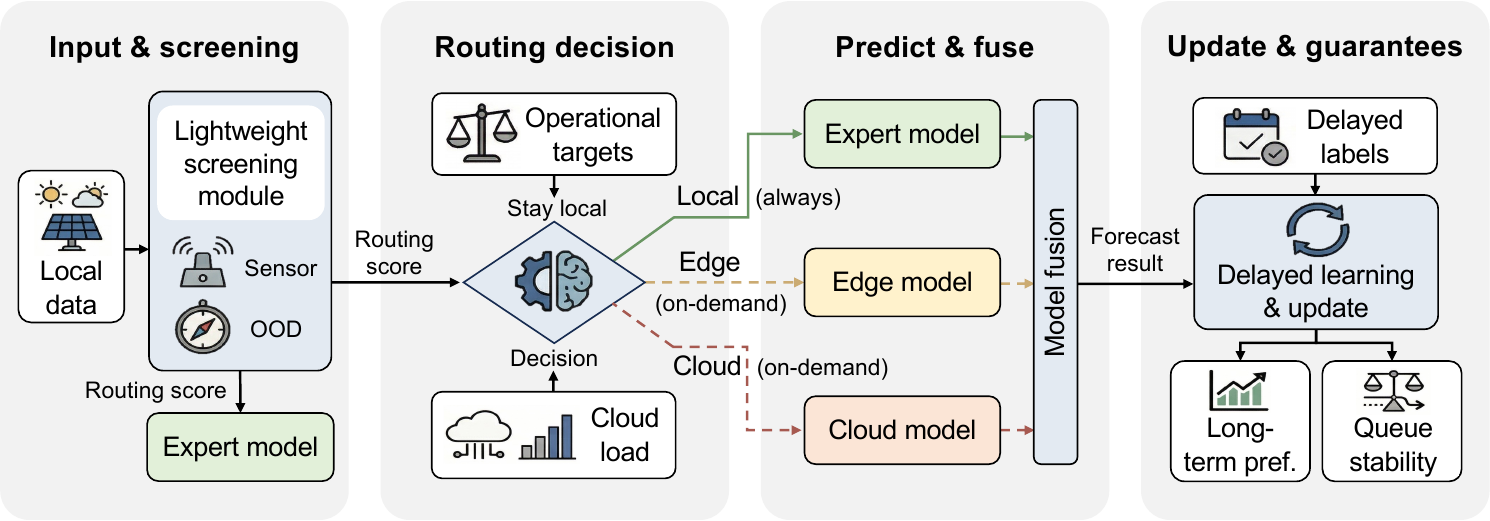}}
\caption{Overview of \emph{CAPE} framework.}
\label{fig:technical_framework}
\vspace{-10pt}
\end{figure}

\subsection{Local Screening and Routing-Score Computation}
At slot \(t\), node \(i\) first computes the default expert prediction
\begin{equation}
\hat{\mathbf{y}}_{i,t}^{e}
=
f_i^e\left(\mathbf{X}_{i,t}^{\mathrm{loc}}\right).
\label{eq:expert_branch_main}
\end{equation}
The routing stage uses only lightweight screening signals that are available before any optional branch escalation. Specifically, the edge forms
\begin{equation}
\boldsymbol{\varphi}_{i,t}
=
\left[
u_{i,t},\,o_{i,t},\,\mu_{i,t},\,d_{i,t}^{\mathrm{scr}}
\right]^\top,
\label{eq:screen_feature_main}
\end{equation}
where \(u_{i,t}\) denotes predictive uncertainty, \(o_{i,t}\) denotes an out-of-distribution score, \(\mu_{i,t}\) denotes weather mutation intensity, and \(d_{i,t}^{\mathrm{scr}}\) denotes a cheap screening disagreement score available before routing. The scalar routing score is
\begin{equation}
r_{i,t}
=
\sigma\left(
\boldsymbol{\beta}^{\top}\boldsymbol{\varphi}_{i,t}
\right)\in(0,1).
\label{eq:routing_score_main}
\end{equation}
\red{In implementation, \(u_{i,t}\) is the predictive variance of the small model under stochastic forward passes, \(o_{i,t}\) is a Mahalanobis OOD score computed from recent in-domain features, \(\mu_{i,t}\) is the normalized short-window variation of meteorological covariates, and \(d_{i,t}^{\mathrm{scr}}\) is the absolute disagreement between the expert and small-model forecasts. We use \(\tilde r_{i,t}=\alpha r_{i,t}\) for gain fitting and CDF maintenance, and fit \(\boldsymbol{\beta}\) offline by logistic regression on a full-branch replay set.}
All quantities in this part are computed locally, and their computational overhead is absorbed into the expert-side execution budget. To avoid unnecessary local overhead in expert-only mode, the edge small model\blue{, the cloud retriever invocation, and the lightweight conditional regressor} are generated on demand only after the routing action has been determined.

\begin{algorithm}[htbp]
\caption{Local Screening and Routing-Score Estimation}
\label{alg:local_screen_main}
\begin{algorithmic}[1]
\REQUIRE Local observation window \(\mathbf{X}_{i,t}^{\mathrm{loc}}\)
\STATE Compute \(\hat{\mathbf{y}}_{i,t}^{e}\) by Eq.~\eqref{eq:expert_branch_main}
\STATE Estimate \(u_{i,t},o_{i,t},\mu_{i,t},d_{i,t}^{\mathrm{scr}}\) and form \(\boldsymbol{\varphi}_{i,t}\) by Eq.~\eqref{eq:screen_feature_main}
\STATE Compute routing score \(r_{i,t}\) by Eq.~\eqref{eq:routing_score_main}
\RETURN \(\hat{\mathbf{y}}_{i,t}^{e},\boldsymbol{\varphi}_{i,t},r_{i,t}\)
\end{algorithmic}
\end{algorithm}

Algorithm~\ref{alg:local_screen_main} is executed entirely on the edge and does not require cloud communication. Its role is to produce the default expert forecast together with a compact screening summary for the routing stage.

\subsection{Lyapunov-Guided Routing under Mean-Field Cloud Congestion}
For each \(a\in\{0,1,2\}\), the routing stage maintains the conditional surrogate loss
\begin{equation}
\widehat{\ell}_{i,t}^{(a)}
=
\widehat{\mathbb{E}}
\left[
\ell_{i,t}^{(a)}
\,\middle|\,
\mathcal{H}_{i,t}
\right],
\qquad a\in\{0,1,2\},
\label{eq:surrogate_loss_main}
\end{equation}
where \(\mathcal{H}_{i,t}\) denotes the pre-decision information set generated by the local observation window, the screening outputs, the current queues, and the currently available online-fusion statistics; \(\ell_{i,t}^{(a)}\) denotes the forecasting loss that would be incurred if mode \(a\) were activated at slot \(t\). These surrogates are calibrated from delayed labeled data. \blue{For the cloud-assisted mode, \(\widehat{\ell}_{i,t}^{(2)}\) is calibrated from the realized loss of the factorized retrieval--prediction branch, so the outer routing rule remains unchanged after the internal decomposition of mode \(2\).}
\red{Because only the selected mode is executed online, \(\widehat{\ell}_{i,t}^{(a)}\), \(G_{1,i}\), and \(G_{2,i}\) are initialized from an offline full-branch replay set in which modes \(0,1,2\) are all evaluated for each sample, thereby avoiding missing counterfactual labels during calibration. In practice, \(G_{1,i}\) and \(G_{2,i}\) are fitted by isotonic regression on the replayed loss gaps.}
We model the expected benefit of using more expensive modes through node-calibrated monotone gain functions:
\begin{equation}
G_{1,i}(r)
=
\mathbb{E}
\left[
\widehat{\ell}_{i,t}^{(0)}-\widehat{\ell}_{i,t}^{(1)}
\,\middle|\,
r_{i,t}=r
\right],
\label{eq:G1_main}
\end{equation}
\begin{equation}
G_{2,i}(r)
=
\mathbb{E}
\left[
\widehat{\ell}_{i,t}^{(1)}-\widehat{\ell}_{i,t}^{(2)}
\,\middle|\,
r_{i,t}=r
\right].
\label{eq:G2_main}
\end{equation}
Continuous, nondecreasing gain curves on \([0,1]\) are natural here, since higher routing scores should not reduce the value of escalation in expectation. With this regularity in place,
\begin{equation}
\kappa_{i,t}^{(1)}
=
Q_{\tau}(t)\big(\tau_i^{(1)}-\tau_i^{(0)}\big),
\label{eq:kappa1_main}
\end{equation}
and
\begin{equation}
\kappa_{i,t}^{(2)}\big(\rho\big)
=
Q_{\tau}(t)\big(\tau_i^{(2)}(\rho)-\tau_i^{(1)}\big)
+
Q_c(t)\kappa_i
+
Q_{\rho}(t),
\label{eq:kappa2_main}
\end{equation}
where \(\rho\in[0,1]\) is a cloud-load estimate. The relative routing indices are
\begin{equation}
J_{i,t}^{(0)} = 0,
\label{eq:J0_main}
\end{equation}
\begin{equation}
J_{i,t}^{(1)}(r)
=
\frac{\kappa_{i,t}^{(1)}}{V}
-
G_{1,i}(r),
\label{eq:J1_main}
\end{equation}
\begin{equation}
J_{i,t}^{(2)}(r,\rho)
=
\frac{\kappa_{i,t}^{(1)}+\kappa_{i,t}^{(2)}(\rho)}{V}
-
G_{1,i}(r)-G_{2,i}(r).
\label{eq:J2_main}
\end{equation}
\blue{Therefore, the internal factorization of the cloud-assisted branch does not alter the outer routing structure in Eq.~\eqref{eq:J2_main}, it only refines how mode \(2\) is realized after selection.}

\begin{theorem}
% [Node-wise routing-score threshold routing]
\label{thm:routing_main}
For fixed queue state \((Q_{\tau}(t),Q_c(t),Q_{\rho}(t))\) and cloud-load estimate \(\rho\in[0,1]\), the action that minimizes the conditional surrogate routing objective is
\begin{equation}
a_{i,t}
=
\arg\min_{a\in\{0,1,2\}}
J_{i,t}^{(a)}.
\label{eq:routing_argmin_main}
\end{equation}
Moreover, for each node \(i\), the three pairwise comparisons
\[
J_{i,t}^{(1)}(r)\le J_{i,t}^{(0)},\quad
J_{i,t}^{(2)}(r,\rho)\le J_{i,t}^{(0)},\quad
J_{i,t}^{(2)}(r,\rho)\le J_{i,t}^{(1)}(r)
\]
are each triggered by a scalar threshold in \(r\). Hence, the optimal action is determined by a node-wise threshold-type partition of the routing-score interval \([0,1]\). Some regions are empty for particular queue states.
\end{theorem}

\begin{proof}
The result follows from pairwise comparison of the three conditional surrogate indices in Eq.~\eqref{eq:J0_main}--\eqref{eq:J2_main} together with the continuity and monotonicity of \(G_{1,i}\) and \(G_{2,i}\). A complete proof is given in Appendix~\ref{app:proof_routing}.
\end{proof}
The node-wise cloud-activation threshold is defined by
\begin{align}
\label{eq:theta_c_main}
&\theta_{c,i}(t,\rho)
\\
=&
\inf
\left\{
r\in[0,1]:
J_{i,t}^{(2)}(r,\rho)
\le
\min\bigl\{
J_{i,t}^{(0)},\,J_{i,t}^{(1)}(r)
\bigr\}
\right\}.\nonumber
\end{align}
Then the mean-field cloud-request ratio satisfies
\begin{equation}
\rho(t)
=
\frac{1}{N}
\sum_{i=1}^{N}
\left[
1-F_{i,t}\left(\theta_{c,i}(t,\rho(t))\right)
\right],
\label{eq:rho_fixed_point_main}
\end{equation}
where \(F_{i,t}(\cdot)\) is a node-specific predictive CDF of the routing score at slot \(t\), estimated from recent streaming data.
\red{In practice, \(F_{i,t}\) is maintained as an exponentially weighted empirical CDF of recent calibrated routing scores.}

\begin{proposition}
% [Node-wise mean-field cloud-load equilibrium]
\label{prop:mfe_main}
Whenever \(F_{i,t}(\cdot)\) and \(\theta_{c,i}(t,\rho)\) are continuous for every \(i\), Eq.~\eqref{eq:rho_fixed_point_main} admits at least one fixed point in \([0,1]\). If, in addition, the induced map is a contraction, then the fixed point is unique.
\end{proposition}

\begin{proof}
Existence follows from Brouwer's fixed-point theorem, and uniqueness follows from Banach's contraction theorem under the stated condition. Details are in Appendix~\ref{app:proof_mfe}.
\end{proof}

Algorithm~\ref{alg:routing_main} is executed centrally at the scheduler. It estimates the current cloud load by fixed-point iteration and then performs Lyapunov-guided routing for all edge nodes.

\begin{algorithm}[htbp]
\caption{Mean-Field Routing at Slot \(t\) (Central Scheduler)}
\label{alg:routing_main}
\begin{algorithmic}[1]
\REQUIRE
Queues \(Q_{\tau}(t),Q_c(t),Q_{\rho}(t)\),
predictive CDF estimates \(\{\widehat{F}_{i,t}\}_{i\in\mathcal{N}}\),
and
routing scores \(\{r_{i,t}\}_{i\in\mathcal{N}}\)
\STATE Initialize \(\rho^{(0)}(t)\) from the previous slot
\FOR{$n=0,\ldots,N_{\mathrm{fp}}-1$}
    \FOR{each node \(i\in\mathcal{N}\)}
        \STATE Compute \(\theta_{c,i}(t,\rho^{(n)}(t))\) by Eq.~\eqref{eq:theta_c_main}
    \ENDFOR
    \STATE Update $\rho^{(n+1)}(t)$
    % \[
    % \rho^{(n+1)}(t)
    % =
    % \frac{1}{N}
    % \sum_{i=1}^{N}
    % \left[
    % 1-\widehat{F}_{i,t}\left(\theta_{c,i}(t,\rho^{(n)}(t))\right)
    % \right]
    % \]
    according to Eq.~\eqref{eq:rho_fixed_point_main}
\ENDFOR
\STATE Set \(\rho(t)=\rho^{(N_{\mathrm{fp}})}(t)\)
\FOR{each node \(i\in\mathcal{N}\)}
    \STATE Compute \(J_{i,t}^{(0)},J_{i,t}^{(1)},J_{i,t}^{(2)}\) by Eq.~\eqref{eq:J0_main}--\eqref{eq:J2_main}
    \STATE Set \(a_{i,t}=\arg\min_{a\in\{0,1,2\}}J_{i,t}^{(a)}\)
\ENDFOR
\RETURN \(\rho(t)\) and \(\{a_{i,t}\}_{i\in\mathcal{N}}\)
\end{algorithmic}
\end{algorithm}

Algorithm~\ref{alg:routing_main} can be interpreted as a dynamic pricing mechanism: when the queues or the predicted cloud congestion increase, the effective cost of cloud usage rises automatically through Eq.~\eqref{eq:kappa1_main}--\eqref{eq:kappa2_main}, thereby pushing more samples back to local inference.

\subsection{Fusion of Selected Predictors with Online Updates}
After routing, the final prediction is produced by combining only the branches activated by the selected mode. \blue{In particular, the cloud-assisted branch is internally decomposed into a cloud-side retrieval stage and an edge-side lightweight conditional prediction stage, while the outer three-branch fusion interface remains unchanged.} If \(a_{i,t}\in\{1,2\}\), the edge computes the small-branch prediction.
\red{To keep notation consistent with Section~\ref{sec:system_model}, the abstract operators \(F^{\mathrm{cld}}\) and \(f^{c}\) are instantiated here as \(F_{\Theta}^{\mathrm{cld}}\) and \(f_{\vartheta}^{c}\), respectively.}
\begin{equation}
\hat{\mathbf{y}}_{i,t}^{s}
=
f^s\left(\mathbf{X}_{i,t}^{\mathrm{loc}}\right).
\label{eq:small_branch_main}
\end{equation}
If \(a_{i,t}=2\), \blue{the cloud-assisted branch is executed as a factorized retrieval--prediction pipeline.} The edge forms a query
\begin{equation}
\mathbf{q}_{i,t}
=
h\left(\mathbf{X}_{i,t}^{\mathrm{loc}}\right),
\label{eq:query_main}
\end{equation}
\blue{and sends it to the cloud retriever, which returns} a support set
\begin{equation}
\mathcal{S}_{i,t}
=
\red{F_{\Theta}^{\mathrm{cld}}}\left(
\mathbf{q}_{i,t},
\mathcal{D}^{\mathrm{cld}},
K
\right),
\label{eq:retrieval_main}
\end{equation}
extracts a cloud context
\begin{equation}
\mathbf{z}_{i,t}
=
\psi\left(\mathcal{S}_{i,t}\right),
\label{eq:context_main}
\end{equation}
and computes the cloud-assisted candidate
\begin{equation}
\hat{\mathbf{y}}_{i,t}^{c}
=
\red{f_{\vartheta}^{c}}\left(
\mathbf{X}_{i,t}^{\mathrm{loc}},
\mathbf{z}_{i,t}
\right).
\label{eq:cloud_branch_main}
\end{equation}
\red{Equations~\eqref{eq:query_main}--\eqref{eq:cloud_branch_main} instantiate the factorized cloud-assisted branch: \(F_{\Theta}^{\mathrm{cld}}\) is a cloud-scale large retrieval model, whereas \(f_{\vartheta}^{c}\) is a compact conditional regressor executed locally after context download. In a parameter-sharing implementation, \(f_{\vartheta}^{c}\) may share its backbone with the edge small model and differ only through the additional retrieved context \(\mathbf{z}_{i,t}\).}

\begin{proposition}[\blue{Retrieval-quality dominated loss gap}]
\label{prop:retrieval_gap_main}
\red{Let \(\mathbf{z}_{i,t}^{\star}\) be an oracle context and let \(f^{c,\star}(\mathbf{X},\mathbf{z})\) denote the corresponding oracle retrieval-assisted predictor. Suppose that \(\ell(\mathbf{y},\hat{\mathbf{y}})\) is \(L_{\ell}\)-Lipschitz in \(\hat{\mathbf{y}}\) on \(\mathcal{Y}\), that \(f^{c,\star}(\mathbf{X},\mathbf{z})\) is \(L_z\)-Lipschitz in \(\mathbf{z}\), and that the lightweight predictor \(f_{\vartheta}^{c}\) satisfies}
\begin{equation}
\sup_{\mathbf{X},\mathbf{z}}
\left\|
\red{f_{\vartheta}^{c}}(\mathbf{X},\mathbf{z})
-
\red{f^{c,\star}}(\mathbf{X},\mathbf{z})
\right\|
\le
\blue{\varepsilon_{\mathrm{pred}}}.
\label{eq:pred_approx_main}
\end{equation}
\blue{Then, for any retrieved context \(\mathbf{z}_{i,t}\),}
\begin{align}
&\left|
\ell\left(
\mathbf{y}_{i,t}^{*},
\red{f_{\vartheta}^{c}}\left(
\mathbf{X}_{i,t}^{\mathrm{loc}},
\mathbf{z}_{i,t}
\right)
\right)
-
\ell\left(
\mathbf{y}_{i,t}^{*},
\red{f^{c,\star}}\left(
\mathbf{X}_{i,t}^{\mathrm{loc}},
\mathbf{z}_{i,t}^{\star}
\right)
\right)
\right|\nonumber\\
\le&
\blue{L_{\ell}
\left(
\varepsilon_{\mathrm{pred}}
+
L_z
\left\|
\mathbf{z}_{i,t}-\mathbf{z}_{i,t}^{\star}
\right\|
\right)}.
\label{eq:retrieval_gap_bound_main}
\end{align}
\end{proposition}

\begin{proof}
\red{The result follows from the Lipschitz continuity of the loss, the Lipschitz continuity of \(f^{c,\star}\) in the retrieved context, and the triangle inequality. A proof is given in Appendix~\ref{app:proof_retrieval_gap}.}
\end{proof}

\blue{Proposition~\ref{prop:retrieval_gap_main} formalizes that, once the retrieved context is sufficiently informative, the remaining prediction burden can be handled by a compact conditional regressor. The excess loss is then controlled additively by the retrieval mismatch and the residual predictor approximation error.}
The active branch sets are
\begin{equation}
\mathcal{M}_0=\{e\},\qquad
\mathcal{M}_1=\{e,s\},\qquad
\mathcal{M}_2=\{e,s,c\}.
\label{eq:active_sets_main}
\end{equation}
For mode \(a\in\{1,2\}\), the fusion loss is
\begin{equation}
g_{i,t}^{(a)}(\mathbf{w})
=
\ell\left(
\mathbf{y}_{i,t}^{*},
\sum_{m\in\mathcal{M}_a}w_m\hat{\mathbf{y}}_{i,t}^{m}
\right),
\quad
\mathbf{w}\in\Delta_{|\mathcal{M}_a|}.
\label{eq:fusion_loss_main}
\end{equation}

\begin{lemma}
% [Convexity of mode-wise fusion]
\label{lem:convex_fusion_main}
For every fixed mode \(a\in\{1,2\}\), the function \(g_{i,t}^{(a)}(\mathbf{w})\) is convex in \(\mathbf{w}\) whenever \(\hat{\mathbf{y}}\mapsto \ell(\mathbf{y}_{i,t}^{*},\hat{\mathbf{y}})\) remains convex over the convex hull generated by the active candidates \(\{\hat{\mathbf{y}}_{i,t}^{m}\}_{m\in\mathcal{M}_a}\). This local condition is satisfied by the standard forecasting losses considered here.
\end{lemma}

\begin{proof}
The fused prediction \(\sum_{m\in\mathcal{M}_a}w_m\hat{\mathbf{y}}_{i,t}^{m}\) remains in the convex hull generated by the active candidates and is affine in \(\mathbf{w}\). Convexity of \(\ell\) on that set therefore implies convexity of \(g_{i,t}^{(a)}(\mathbf{w})\).
\end{proof}

To implement online fusion, we adopt an entropic FTRL update with a calibrated strictly positive prior \(\bar{\boldsymbol{\pi}}_i^{(a)}\), i.e., \(\bar{\pi}_{i,m}^{(a)}>0\) for every \(m\in\mathcal{M}_a\). Since multi-horizon labels are revealed with delay, the set of revealed mode-\(a\) samples available by slot \(t\) is
\(
\mathcal{R}_{i,a}(t)
=
\left\{
s<t:
a_{i,s}=a,\;
\mathbf{y}_{i,s}^{*}\ \text{has been fully revealed by slot } t
\right\},
\)
and the cumulative revealed subgradient is
\begin{equation}
\boldsymbol{\Gamma}_{i,t-1}^{(a)}
=
\sum_{s\in\mathcal{R}_{i,a}(t)}
\mathbf{g}_{i,s}^{(a)},
\label{eq:gamma_cum_main}
\end{equation}
where \(\mathbf{g}_{i,s}^{(a)}\in\partial g_{i,s}^{(a)}(\mathbf{w}_{i,s}^{(a)})\). For \(a\in\{1,2\}\), the action-specific fusion weight is
\begin{align}
\mathbf{w}_{i,t}^{(a)}
=
\arg\min_{\mathbf{w}\in\Delta_{|\mathcal{M}_a|}}
\left\{
\eta\left\langle\boldsymbol{\Gamma}_{i,t-1}^{(a)},\mathbf{w}\right\rangle
+
D_{\mathrm{KL}}\left(
\mathbf{w}\,\|\,\bar{\boldsymbol{\pi}}_i^{(a)}
\right)
\right\}.
\label{eq:ftrl_main}
\end{align}

\begin{lemma}
% [Closed-form fusion update]
\label{lem:closed_form_main}
For \(m\in\mathcal{M}_a\), the solution of Eq.~\eqref{eq:ftrl_main} is
\begin{equation}
w_{i,m}^{(a)}(t)
=
\frac{
\bar{\pi}_{i,m}^{(a)}
\exp\left(-\eta\Gamma_{i,m}^{(a)}(t-1)\right)
}{
\sum_{n\in\mathcal{M}_a}
\bar{\pi}_{i,n}^{(a)}
\exp\left(-\eta\Gamma_{i,n}^{(a)}(t-1)\right)
}.
\label{eq:omd_closed_form_main}
\end{equation}
\end{lemma}

\begin{proof}
The statement follows from the KKT conditions of Eq.~\eqref{eq:ftrl_main}. The complete derivation is given in Appendix~\ref{app:proof_closed_form}.
\end{proof}

\begin{theorem}
% [Fusion regret on the revealed-loss sequence]
\label{thm:regret_main}
Fix \(i\) and \(a\in\{1,2\}\), and write \(R\) for the number of revealed mode-\(a\) updates processed by the online fusion module. When the subgradients of \(g_{i,t}^{(a)}\) remain uniformly bounded, the entropic FTRL updater induced by Eq.~\eqref{eq:ftrl_main} achieves
% \begin{equation}
$\mathrm{Regret}_{i,a}(R)
=
O\left(\sqrt{R}\right)$
% \end{equation}
against the best fixed weight vector in hindsight on that revealed-loss sequence.
\end{theorem}

\begin{proof}
The result follows from standard entropic FTRL analysis on the simplex together with Lemma~\ref{lem:convex_fusion_main}, applied to the revealed-loss sequence. The proof is given in Appendix~\ref{app:proof_regret}.
\end{proof}

Algorithm~\ref{alg:fusion_main} is executed locally for each node after the routing action has been determined. \blue{It performs cloud retrieval and lightweight conditional prediction only when necessary,} computes the mode-wise fused forecast, and stores unresolved samples until their labels become available.

\begin{algorithm}[h]
\caption{Fusion of Selected Predictors with Online Updates}
\label{alg:fusion_main}
\begin{algorithmic}[1]
\REQUIRE Routing action \(a_{i,t}\), cloud-load estimate \(\rho(t)\), expert candidate \(\hat{\mathbf{y}}_{i,t}^{e}\), local observation window \(\mathbf{X}_{i,t}^{\mathrm{loc}}\)
\IF{\(a_{i,t}\in\{1,2\}\)}
    \STATE Compute \(\hat{\mathbf{y}}_{i,t}^{s}\) by Eq.~\eqref{eq:small_branch_main}
\ENDIF
\IF{\(a_{i,t}=2\)}
    \STATE Form \(\mathbf{q}_{i,t}\) by Eq.~\eqref{eq:query_main}
    \STATE Retrieve \(\mathcal{S}_{i,t}\) \blue{via the cloud retriever} by Eq.~\eqref{eq:retrieval_main}
    \STATE Form \(\mathbf{z}_{i,t}\) by Eq.~\eqref{eq:context_main} and compute \(\hat{\mathbf{y}}_{i,t}^{c}\) \blue{via the lightweight conditional regressor} by Eq.~\eqref{eq:cloud_branch_main}
\ENDIF
\IF{\(a_{i,t}=0\)}
    \STATE Output \(\hat{\mathbf{y}}_{i,t}=\hat{\mathbf{y}}_{i,t}^{e}\)
\ELSE
    \STATE Select active set \(\mathcal{M}_{a_{i,t}}\) from Eq.~\eqref{eq:active_sets_main}
    \STATE Compute \(\mathbf{w}_{i,t}^{(a_{i,t})}\) by Eq.~\eqref{eq:omd_closed_form_main}
    \STATE Output \(\hat{\mathbf{y}}_{i,t}=\sum_{m\in\mathcal{M}_{a_{i,t}}}w_{i,m}^{(a_{i,t})}(t)\hat{\mathbf{y}}_{i,t}^{m}\)
\ENDIF
\STATE Store the necessary local state for slot \(t\) in a buffer until the corresponding label is fully revealed
\FOR{each sample \(s\) whose \(\mathbf{y}_{i,s}^{*}\) becomes available at \(t\)}
    \IF{\(a_{i,s}\in\{1,2\}\)}
        \STATE Compute \(\mathbf{g}_{i,s}^{(a_{i,s})}\in\partial g_{i,s}^{(a_{i,s})}(\mathbf{w}_{i,s}^{(a_{i,s})})\)
        \STATE Update \(\boldsymbol{\Gamma}_{i,t}^{(a_{i,s})}\leftarrow \boldsymbol{\Gamma}_{i,t-1}^{(a_{i,s})}+\mathbf{g}_{i,s}^{(a_{i,s})}\)
    \ENDIF
    \STATE \red{Update the executed-mode calibrator used in Eq.~\eqref{eq:surrogate_loss_main}}
\ENDFOR
\RETURN \(\hat{\mathbf{y}}_{i,t}\), \(\tau_i^{(a_{i,t})}(\rho(t))\), \(c_i^{(a_{i,t})}\), and \(\mathbb{I}\{a_{i,t}=2\}\)
\end{algorithmic}
\end{algorithm}

After Algorithm~\ref{alg:fusion_main} has been executed for all \(i\in\mathcal{N}\), the scheduler aggregates the returned per-node resource usages and performs one global queue update by Eq.~\eqref{eq:q_tau}--\eqref{eq:q_rho}. Algorithms~\ref{alg:local_screen_main}--\ref{alg:fusion_main} together form the complete online pipeline.

\subsection{Long-Term Performance Guarantee}
The realized one-slot average forecasting loss is
\begin{equation}
L(t)
=
\frac{1}{N}\sum_{i=1}^{N}\ell_{i,t}.
\label{eq:slot_loss_main}
\end{equation}
In the long-term analysis, the implemented controller is allowed an additive inexactness level \(\delta\) in the right-hand side of the one-slot drift-plus-penalty upper bound. \blue{Under the factorized cloud-assisted implementation, this constant aggregates the surrogate-loss calibration error, the finite fixed-point iteration error, the retrieval mismatch error, and the residual approximation error of the lightweight conditional regressor.}
\blue{Under Proposition~\ref{prop:retrieval_gap_main}, the factorized cloud-assisted branch yields the following concrete refinement.}

\begin{proposition}
% [\blue{Factorized inexactness of the cloud-assisted mode}]
\label{prop:delta_decomp_main}
\blue{If the retrieved context error is uniformly bounded as}
\begin{equation}
\left\|
\mathbf{z}_{i,t}-\mathbf{z}_{i,t}^{\star}
\right\|
\le
{\varepsilon_{\mathrm{ret}}},
\forall i,t,
\label{eq:ret_uniform_main}
\end{equation}
\blue{then admissible instantiation of the per-slot additive inexactness in Theorem~\ref{thm:lyapunov_main} is}
\begin{equation}
\delta
=
\blue{\delta_{\mathrm{sur}}
+
\delta_{\mathrm{fp}}
+
V L_{\ell}
\left(
\varepsilon_{\mathrm{pred}}
+
L_{z}\varepsilon_{\mathrm{ret}}
\right)},
\label{eq:delta_decomp_main}
\end{equation}
\blue{where \(\delta_{\mathrm{sur}}\) and \(\delta_{\mathrm{fp}}\) denote the contributions of surrogate-loss calibration and finite fixed-point iteration, respectively.}
\end{proposition}

\begin{proof}
\blue{The result follows by applying Proposition~\ref{prop:retrieval_gap_main} to the cloud-assisted branch, observing that the induced loss perturbation is scaled by \(V\) in the drift-plus-penalty objective, and then adding the surrogate-calibration and fixed-point contributions. Full details are deferred to Appendix~\ref{app:proof_delta_decomp}.}
\end{proof}

\begin{theorem}
% [Long-term performance under approximate per-slot control]
\label{thm:lyapunov_main}
When the one-slot loss/resource consumption and one-slot queue increments are bounded, the Slater condition holds, and the implemented controller minimizes the right-hand side of the one-slot drift-plus-penalty upper bound within an additive error \(\delta\) at every slot, the proposed Lyapunov-guided controller satisfies
\begin{equation}
\bar{L}^{\mathrm{alg}}
\le
L^{*}+\frac{B_0+\delta}{V},
\label{eq:longterm_loss_main}
\end{equation}
and its average queue backlog is \(O(V)\). Therefore, the virtual queues are mean-rate stable, and the time-average latency, communication, and cloud-usage constraints are satisfied.
\end{theorem}

\begin{proof}
The proof follows from the standard drift-plus-penalty argument using the queue updates in Eq.~\eqref{eq:q_tau}--\eqref{eq:q_rho} together with the per-slot additive error bound \(\delta\). Full details are deferred to Appendix~\ref{app:proof_lyapunov}.
\end{proof}

% \clearpage
\begin{table*}[!t]
\tiny
\centering
\caption{
% Performance comparison of different forecasting methods and system baselines on the Hunan and Shanxi datasets. Lower is better for nMAE, nRMSE, REE, and DG, while higher is better for AUROC and AUPRC. All error metrics are reported as full-scale normalized errors (\%FS). AUROC/AUPRC are only reported for adaptive-routing methods, while fixed-strategy baselines are marked as ``--''.
Performance comparison of different forecasting methods and system baselines on the Hunan and Shanxi datasets. Lower is better for nMAE, nRMSE, REE, and DG, while higher is better for AUROC and AUPRC. All error metrics are reported as full-scale normalized errors (\%FS). AUROC/AUPRC are only reported for adaptive-routing methods, while fixed-strategy baselines are marked as ``--'' since these metrics are not applicable to methods without routing decisions. The best results are shown in \textbf{bold}, and the second-best results are \underline{underlined}.}
\label{tab:performance_hunan_shanxi}
\setlength{\tabcolsep}{4.0pt}
\renewcommand{\arraystretch}{1.12}
\resizebox{0.9\textwidth}{!}{
\begin{tabular}{|c|c|c|c|c|c|c|c|c|c|c|}
\hline
\textbf{Dataset} & \textbf{Metric} & \textbf{ExO} & \textbf{EdO} & \textbf{CO} & \textbf{ACA} & \textbf{STR} & \textbf{Moirai} & \textbf{AIRG} & \textbf{STKD-PV} & \textbf{\textit{CAPE} (Ours)} \\
\hline

\multirow{6}{*}{{Hunan}}
& nMAE (\%FS, $\downarrow$) 
& 8.54 & 6.21 & \underline{3.15} & 4.10 & 4.02 & 3.32 & 5.54 & 4.56 & \textbf{3.08} \\
\cline{2-11}
& nRMSE (\%FS, $\downarrow$) 
& 11.23 & 8.45 & \underline{4.52} & 5.91 & 5.86 & 4.81 & 7.23 & 6.05 & \textbf{4.46} \\
\cline{2-11}
& AUROC ($\uparrow$) 
& -- & -- & -- & \underline{0.762} & 0.755 & -- & 0.685 & 0.710 & \textbf{0.924} \\
\cline{2-11}
& AUPRC ($\uparrow$) 
& -- & -- & -- & \underline{0.654} & 0.648 & -- & 0.590 & 0.625 & \textbf{0.885} \\
\cline{2-11}
& REE (\%FS, $\downarrow$) 
& 25.41 & 18.52 & \underline{8.56} & 11.24 & 11.05 & 9.04 & 14.52 & 12.15 & \textbf{8.45} \\
\cline{2-11}
& DG ($\downarrow$) 
& 3.50 & 2.10 & \underline{1.15} & 1.48 & 1.45 & 1.18 & 1.70 & 1.35 & \textbf{1.12} \\
\hline
\multirow{6}{*}{{Shanxi}}
& nMAE (\%FS, $\downarrow$) 
& 7.85 & 6.92 & \textbf{4.09} & 4.72 & 4.68 & 4.51 & 5.48 & 4.96 & \underline{4.17} \\
\cline{2-11}
& nRMSE (\%FS, $\downarrow$) 
& 10.54 & 9.18 & \textbf{5.94} & 6.41 & 6.47 & 6.26 & 7.43 & 6.88 & \underline{6.02} \\
\cline{2-11}
& AUROC ($\uparrow$) 
& -- & -- & -- & 0.864 & 0.871 & -- & 0.821 & \underline{0.893} & \textbf{0.931} \\
\cline{2-11}
& AUPRC ($\uparrow$) 
& -- & -- & -- & 0.687 & 0.701 & -- & 0.624 & \underline{0.736} & \textbf{0.804} \\
\cline{2-11}
& REE (\%FS, $\downarrow$) 
& 13.40 & 11.10 & \textbf{6.01} & 7.02 & 6.93 & 6.24 & 8.71 & 7.42 & \underline{6.15} \\
\cline{2-11}
& DG ($\downarrow$) 
& 1.89 & 1.58 & \underline{1.09} & 1.16 & 1.15 & 1.10 & 1.33 & 1.19 & \textbf{1.08} \\
\hline
\end{tabular}
}
\vspace{-15pt}
\end{table*}

\section{Experimental Evaluation}
\label{sec:experiments}

We evaluate the proposed condition-adaptive cloud-edge collaborative framework by answering the following questions:
\begin{enumerate}
    \item How does the proposed framework perform on PV power forecasting accuracy, and how effective is its condition-adaptive routing strategy at balancing this accuracy with system constraints compared to baselines?
    \item How robust is the proposed approach under challenging conditions?
    \item What is the effect of critical hyperparameters on the system's overall performance?
\end{enumerate}

\subsection{Experimental Setup}
\label{subsec:expSetup}

We first present the setup used in our experiments, including the datasets, baselines, models, metrics, and parameters.

\paragraph{Datasets and Evaluation Protocol.} \red{We evaluate on two real-world PV datasets from Shanxi and Hunan, China. Shanxi contains telemetry from 18 inverter nodes over 19 days (5-min power, 10-min weather), and Hunan contains five inverter channels with 1-min power telemetry plus coarser site-level solar/weather covariates. Each inverter/channel is treated as one forecasting node; we align exogenous covariates using only the latest past weather record, remove unmatched/incomplete intervals, and trim boundary samples, resulting in 46,983 samples (Shanxi) and 24,929 samples (Hunan). We use strictly chronological train/validation/test splits in each region (fit/tune/test). In-domain evaluation uses the held-out test block of the same region, while cross-region evaluation trains on one region and tests on the other. To prevent temporal leakage, \(\mathcal{D}^{\mathrm{cld}}\) is built from training windows only, retrieval at time \(t\) uses cases ending before \(t\), and a test window is inserted only after its horizon is revealed. In the data-scarce setting, the site expert is trained on a small local prefix, while the shared retriever and cross-site case base still use full training histories from other sites. Hard-subset evaluation focuses on large power ramps or strong weather mutations, and OOD evaluation follows the validation-calibrated router OOD score.}

\paragraph{Baselines.} To isolate different performance gains, we compare our approach against two categories of baselines, including forecasting baselines and system baselines. For assessing overall end-to-end forecasting capability, we utilize strong models, including 
% iTransformer~\cite{liu2024itransformer}, TimeMixer~\cite{wang2024timemixer,Das2024TimesFM}, TimesFM~\cite{Das2024TimesFM}, Moirai~\cite{Woo2024Moirai}, MOMENT~\cite{goswami2024moment}, Time-LLM~\cite{TimeLLM2024}, TabPFN~\cite{Hollmann2025TabPFN}, 
% iTransformer~\cite{liu2024itransformer},
Moirai~\cite{Woo2024Moirai},
AIRG~\cite{AIRG},
and STKD-PV~\cite{he2025stkdpv}. Additionally, to evaluate the effectiveness of collaboration and routing policies, we compare against system-level baselines, namely Expert-only (ExO), Edge-only (EdO), Cloud-only (CO), Always-cloud-assisted (ACA), and Static-threshold routing (STR).
\paragraph{Model Instantiation.} 
The collaborative architecture employs a parameter-efficient temporal convolutional network as the site-specific expert, alongside a lightweight LimiX model for edge-side causal inference and a high-capacity LimiX model with historical retrieval deployed in the cloud \cite{Zhang2025LimiX}. \red{In the data-scarce setting, the edge-side branch replies only on the reduced local data, whereas the globally shared cloud-side retrieval pool remains training-only and cross-site.} 
% Experiments are implemented in PyTorch and executed on a Linux server equipped with NVIDIA GeForce RTX 3090 GPUs.
% \paragraph{Metrics and Parameters:} System-level performance, including latency, communication overhead, cloud usage, and virtual queue backlog, is profiled using a discrete-event simulator. Forecasting accuracy is evaluated via MAE, RMSE, nMAE, and nRMSE, while routing quality is assessed using AUROC and AUPRC. To quantify robustness under distribution shifts, we report ramp event errors (REE) and the out-of-distribution degradation ratio $DG={\text{Error}_\mathrm{OOD}}/{\text{Error}_\mathrm{ID}}$, where error refers to scale error.
% % \paragraph{Default Settings and Parameters:} 
% All hyperparameters and routing thresholds are strictly tuned on the validation set to satisfy the average cloud budgets. In the default configuration, we set the Lyapunov trade-off parameter $V=80$, the routing calibration scalar $\alpha=1.0$, the long-term cloud-query budget $\rho_{\max}=0.5$, the time-average latency budget $\tau_{\max}=120$, the fixed-point iterations $N_{\mathrm{fp}}=5$, and the retrieval support-set size $K=8$. All virtual queues are initialized to zero at the beginning of the evaluation.
% Experiments are implemented in PyTorch and executed on a Linux server equipped with four NVIDIA GeForce RTX 3090 GPUs.

\begin{figure}[htbp]
\centering
    \includegraphics[width=0.90\linewidth]{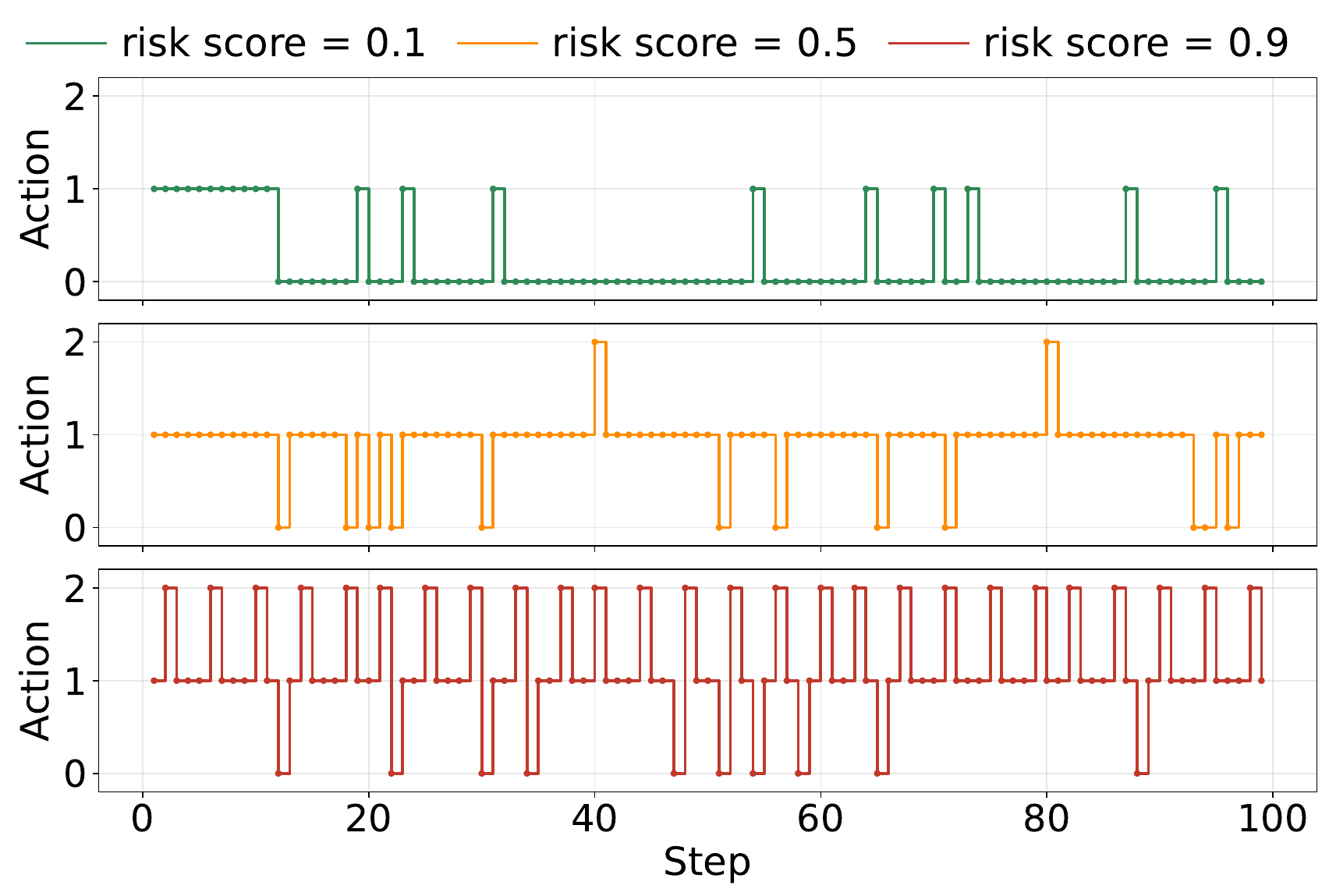}
\vspace{-10pt}\caption{Router action traces at three routing-score levels, where higher scores lead to more frequent escalation to larger models.}\vspace{-10pt}
\label{routing-action}
\end{figure}

\paragraph{Evaluation Metrics and Hyperparameters.} 
System-level performance (latency, communication overhead, cloud usage, and virtual-queue backlog) is measured by a discrete-event simulator. Forecasting accuracy is reported by MAE, RMSE, nMAE, and nRMSE. For adaptive methods, routing quality is reported by AUROC and AUPRC using \(r_{i,t}\) as the ranking score and the oracle cloud-routing decision as the positive label; the oracle label is \(1\) iff mode \(2\) yields the minimum replayed loss among feasible modes. Robustness is evaluated by ramp event error (REE) and the OOD degradation ratio \(DG=\mathrm{Error}_{\mathrm{OOD}}/\mathrm{Error}_{\mathrm{ID}}\), where error is the full-scale normalized forecasting error. OOD/ID partition follows the validation-calibrated OOD threshold, and all methods use the same fixed random seed for preprocessing, splitting, replay construction, and initialization (unless otherwise noted, we report fixed-seed results).
% \paragraph{Default Settings and Parameters:} 
All hyperparameters and routing thresholds are strictly tuned on the validation set to satisfy the average cloud budgets. In the default configuration, we set the Lyapunov trade-off parameter $V=80$, the routing calibration scalar $\alpha=1.0$, the long-term cloud-query budget $\rho_{\max}=0.5$, the time-average latency budget $\tau_{\max}=120$, the fixed-point iterations $N_{\mathrm{fp}}=5$, and the retrieval support-set size $K=8$. All virtual queues are initialized to zero at the beginning of the evaluation.
\red{The scalar \(\alpha\) rescales \(r_{i,t}\) before fitting \(G_{1,i}\), \(G_{2,i}\), and \(F_{i,t}\), and is tuned on the validation set.}

\begin{figure}[htbp]
\centering
    \includegraphics[width=0.98\linewidth]{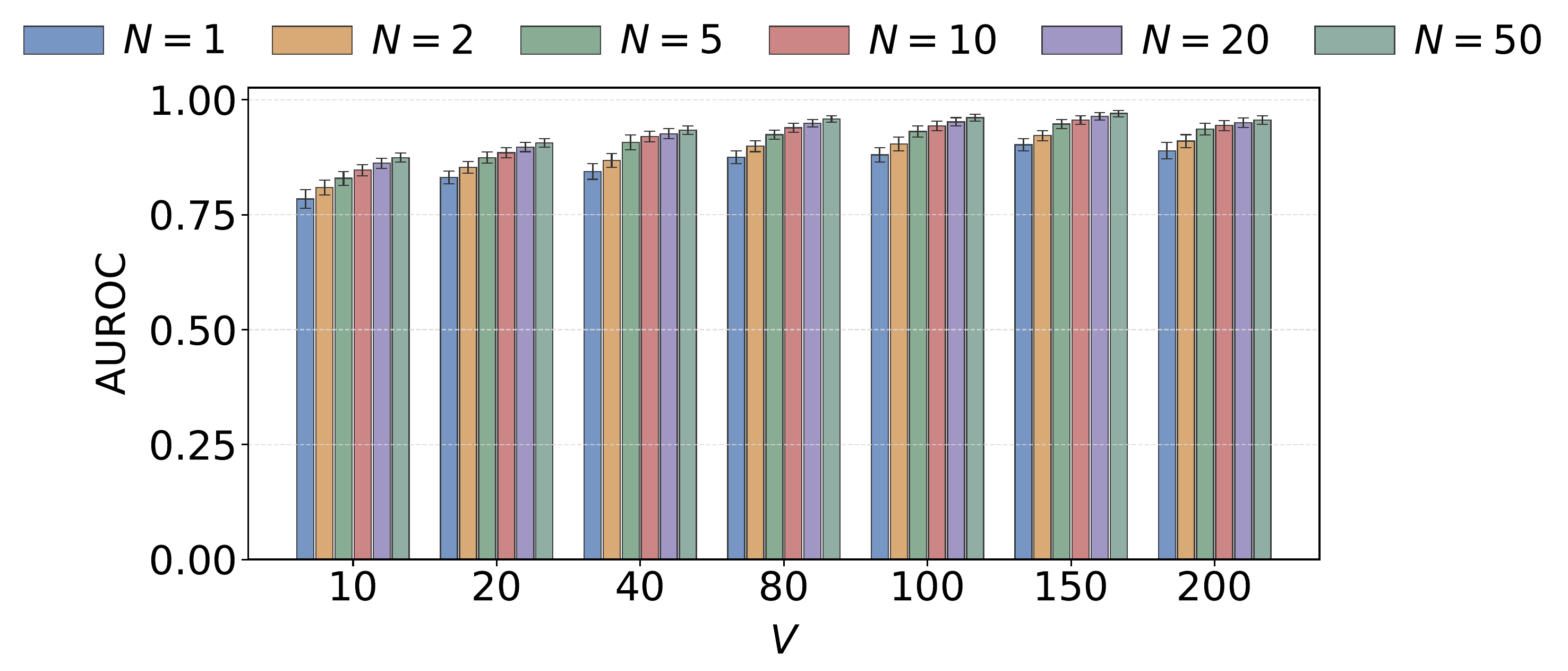}
    \vspace{-10pt}
\caption{\red{Effect of the Lyapunov trade-off parameter $V$ on AUROC under different edge counts $N$.}}
\label{v_plus_n_vs_auroc}
\vspace{-5pt}
\end{figure}

\subsection{Results and Analysis}
\paragraph{Overall Comparison.}
\red{As shown in Table~\ref{tab:performance_hunan_shanxi}, the proposed framework delivers the best overall balance between forecasting accuracy, routing quality, and robustness across the Hunan and Shanxi datasets. On Hunan, our method achieves the best results on all reported metrics, including nMAE, nRMSE, AUROC, AUPRC, REE, and DG, showing that it improves both prediction accuracy and the identification of difficult samples that benefit from cloud assistance. On Shanxi, although the Cloud-only baseline attains the lowest nMAE, nRMSE, and REE, our method remains highly competitive on these error metrics and achieves the best AUROC, AUPRC, and DG. These results indicate that, compared with both fixed-strategy and adaptive-routing baselines, the proposed framework provides a stronger overall trade-off between predictive performance and selective cloud-edge collaboration under system constraints.}

\begin{figure}[t]
\centering
    \includegraphics[width=0.85\linewidth]{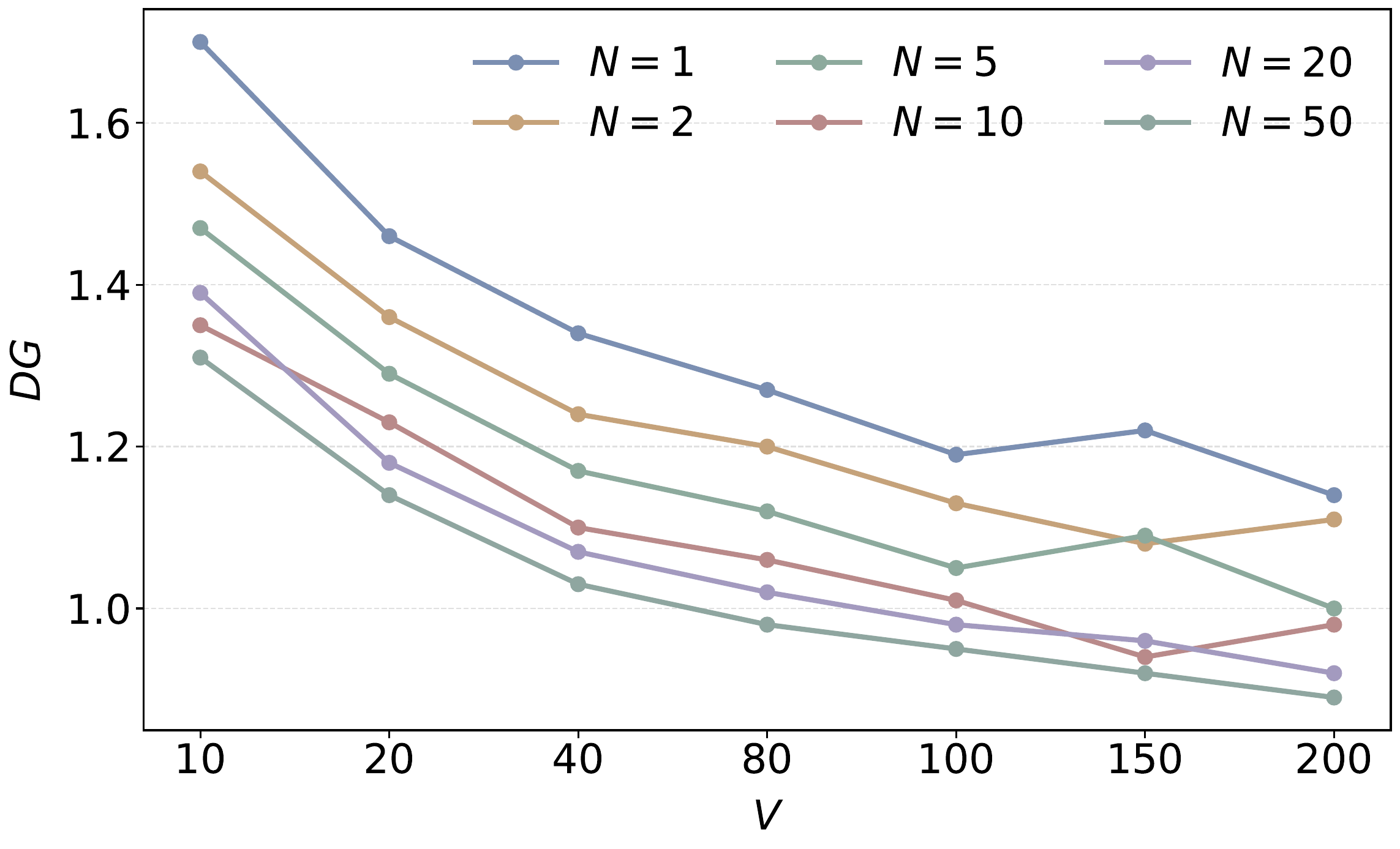}
    \vspace{-10pt}
\caption{\red{Effect of the Lyapunov trade-off parameter $V$ on $DG$ under different edge counts $N$.}}
\label{v_plus_n_vs_dg}
\vspace{-15pt}
\end{figure}

\begin{figure}[htbp]
\centering
    \includegraphics[width=0.98\linewidth]{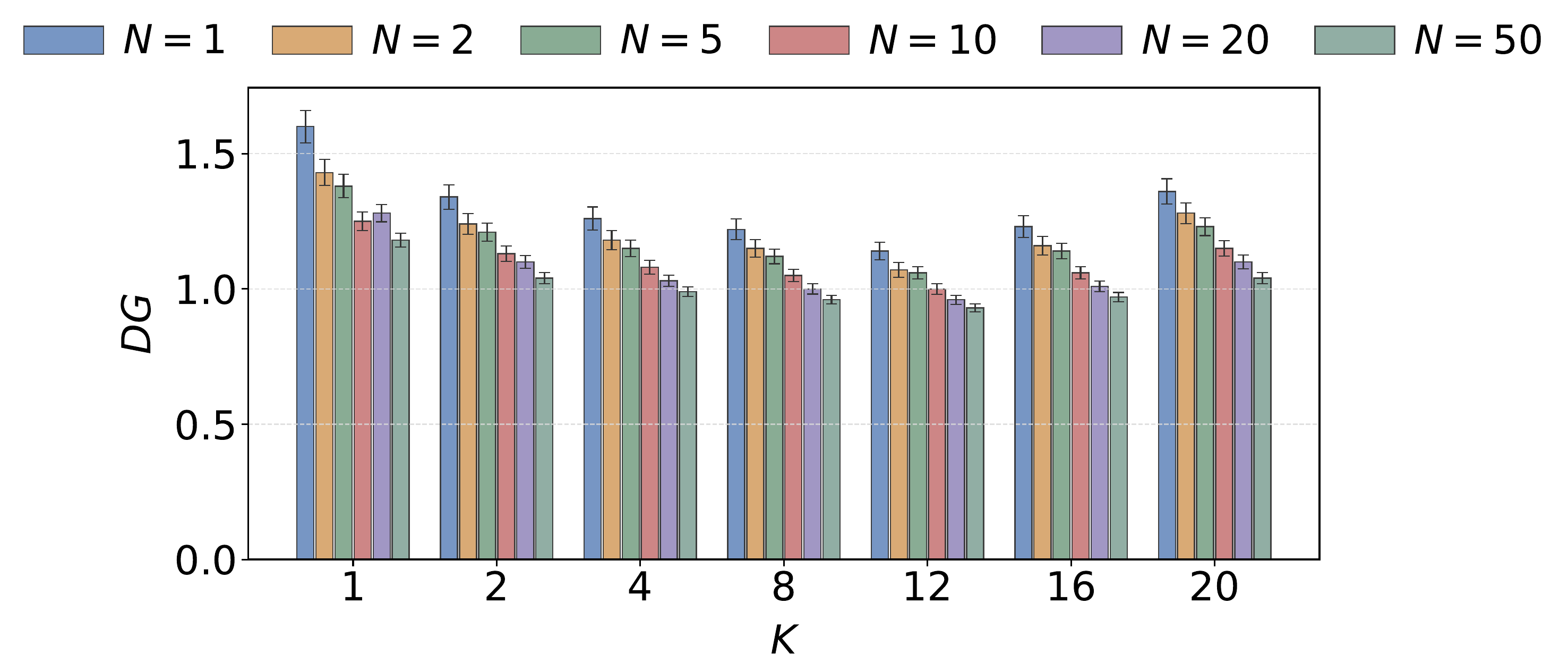}
\vspace{-10pt}\caption{Sensitivity of $DG$ to $K$ under different edge counts $N$.}\vspace{-15pt}
\label{k_plus_n_vs_dg}
\end{figure}
\paragraph{Robustness under Challenging Conditions.}
% \paragraph{Robustness under challenging conditions.}
\red{Beyond average accuracy, the proposed framework also shows strong robustness under challenging conditions. In Table~\ref{tab:performance_hunan_shanxi}, our method achieves the lowest REE and DG on Hunan and the lowest DG with the second-best REE on Shanxi, indicating that it better preserves forecasting quality on difficult and out-of-distribution samples. This robustness is further supported by Fig.~\ref{v_plus_n_vs_auroc} and Fig.~\ref{v_plus_n_vs_dg}, where increasing the Lyapunov trade-off parameter $V$ generally improves AUROC and reduces DG across different edge counts $N$, especially from small to moderate values, suggesting better routing discrimination and stronger robustness to distribution shift. Although mild non-monotonic fluctuations appear at large $V$ in some settings, the overall trend remains favorable. This interpretation is also consistent with the router behavior shown in Fig.~\ref{routing-action}: when the routing score is low, the scheduler mainly keeps inference at the expert level, whereas higher scores trigger more frequent activation of the edge collaborative mode and, when necessary, escalation to the cloud-assisted mode. Such score-dependent transitions indicate that the proposed controller can reserve more expensive resources for harder cases while maintaining stable performance on easier ones.}

\begin{figure}[htbp]
\centering
    \includegraphics[width=0.85\linewidth]{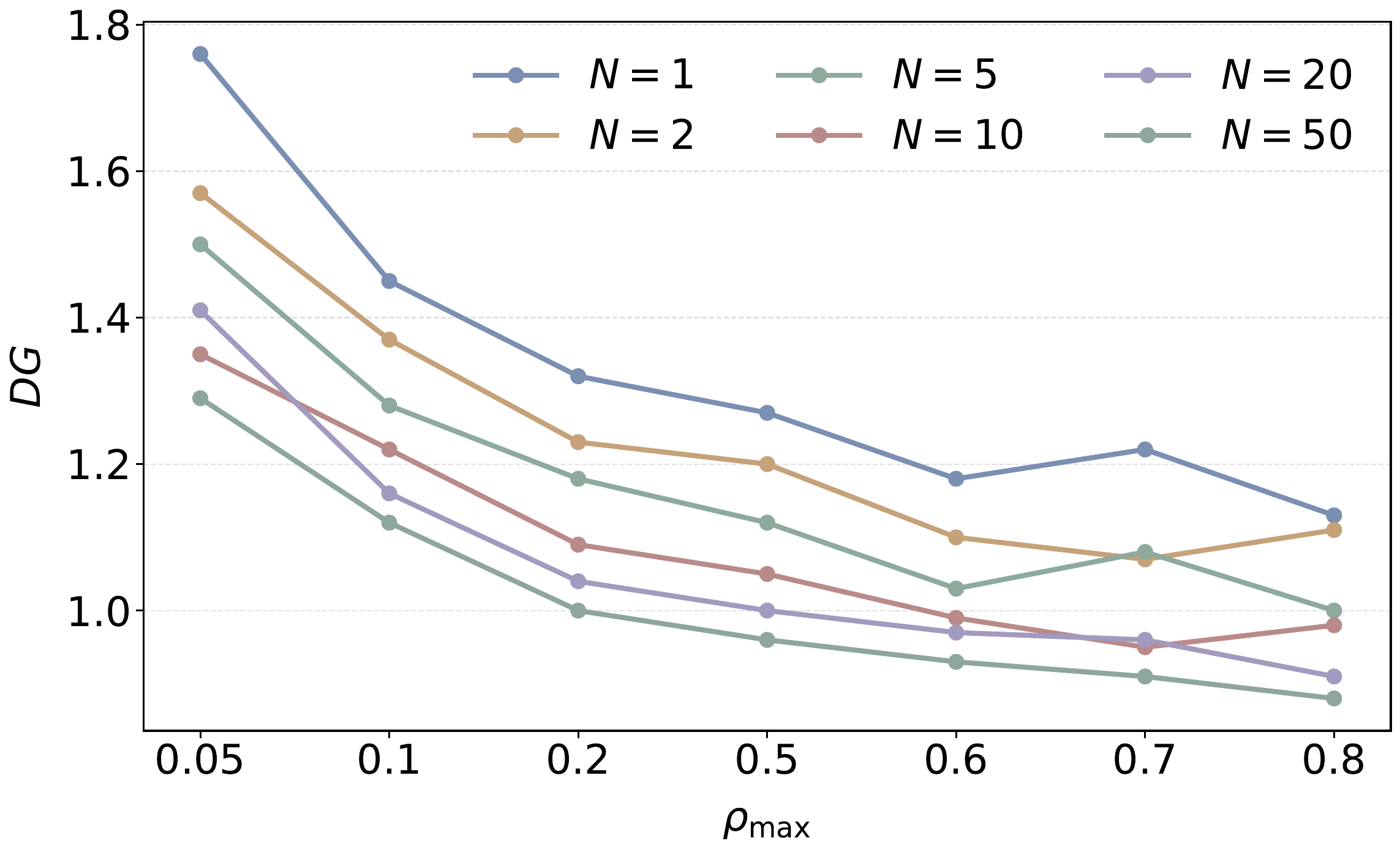}
\vspace{-10pt}\caption{Sensitivity of $DG$ to $\rho_{\max}$ under different edge counts $N$.}\vspace{-10pt}
% \vspace{-5pt}
\label{rho_plus_n_vs_dg}
\end{figure}

\begin{figure}[b]
\centering
    \includegraphics[width=0.85\linewidth]{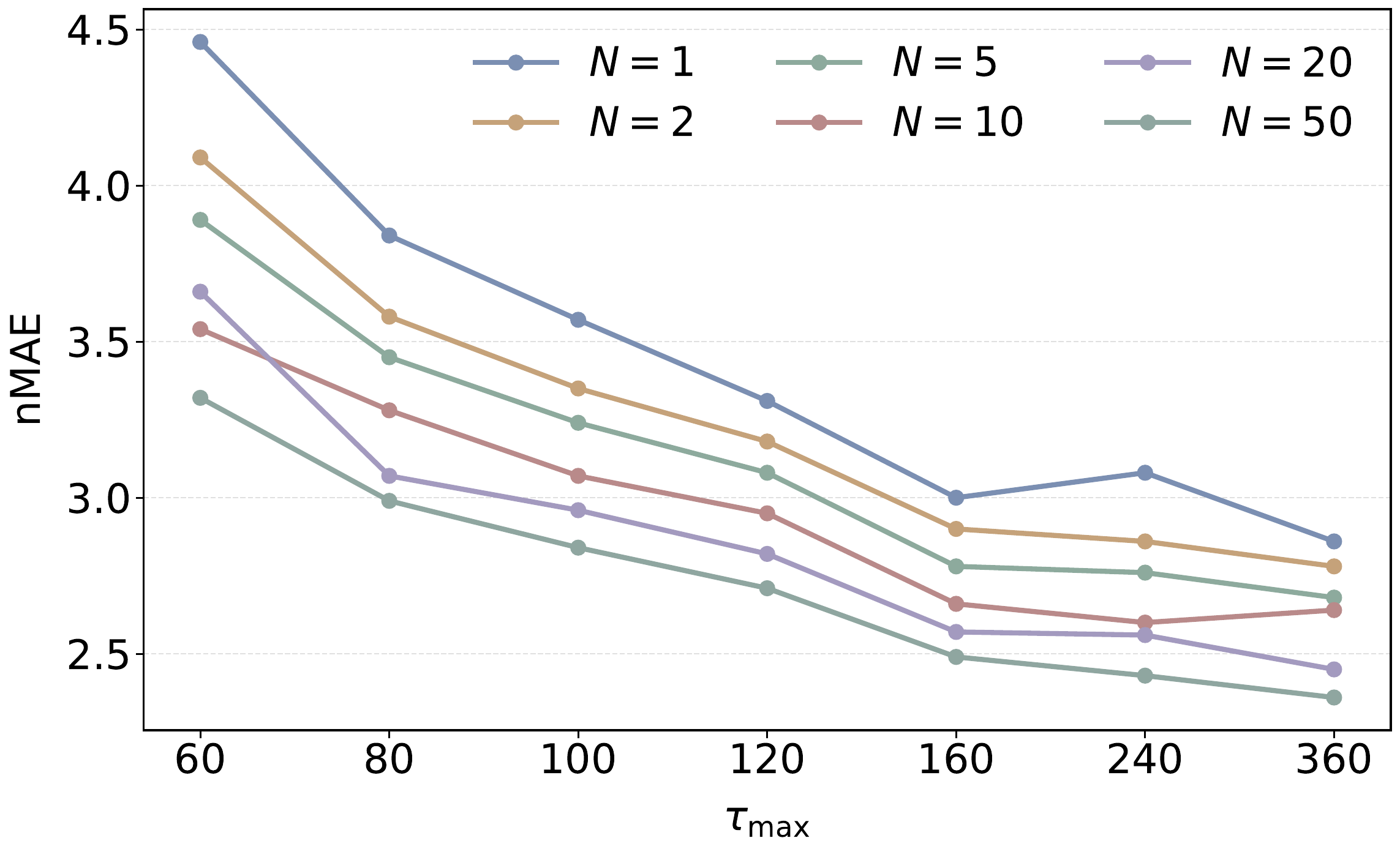}
\vspace{-10pt}\caption{Sensitivity of nMAE to $\tau_{\max}$ under different edge counts $N$.}\vspace{-10pt}
\label{tau_plus_n_vs_nmae}
\end{figure}

\paragraph{Sensitivity to Key Hyperparameters.}
\red{To understand the effect of the remaining critical hyperparameters on overall system performance, we study the sensitivity of the framework to the retrieval support-set size $K$, the cloud budget $\rho_{\max}$, and the latency budget $\tau_{\max}$. Fig.~\ref{k_plus_n_vs_dg} shows that the support-set size $K$ has a clear non-monotonic effect: moderate values, especially around $K=8$ to $12$, yield lower DG by providing sufficiently rich retrieved context, whereas overly small or overly large support sets degrade robustness. As shown in Fig.~\ref{rho_plus_n_vs_dg}, increasing $\rho_{\max}$ generally improves OOD robustness because more difficult samples can be escalated to the cloud-assisted mode, although the improvement gradually saturates at higher budgets. Fig.~\ref{tau_plus_n_vs_nmae} further shows that relaxing the latency budget $\tau_{\max}$ consistently reduces nMAE by allowing stronger collaborative modes to be invoked more often, with diminishing returns as the budget becomes large. Across these studies, larger edge counts $N$ usually lead to better performance, but the gain still depends on the operating region of each hyperparameter. Overall, these suggest that the proposed system is tunable in an interpretable way and benefits more from balanced configurations than from simply maximizing retrieval depth, cloud usage, or latency allowance.}

\section{Conclusion}

In this paper, we propose a condition-adaptive cloud-edge collaborative framework to balance predictive robustness and inference latency in photovoltaic power forecasting. The architecture dynamically integrates a site-specific expert predictor, an edge-side small causal model, and a cloud-side large model. To satisfy long-term latency, communication, and cloud-usage constraints, a screening module selectively invokes cloud-assisted historical retrieval only during challenging scenarios driven by weather mutations or distribution shifts. Ultimately, through confidence-aware fusion, the proposed system effectively leverages large-model priors to achieve high forecasting accuracy under complex conditions while preserving the speed of local edge execution.

% %%@@@@@@@@@@@@@@@@@@@@@@@@@@@@@@@@@@@@@@@@@@@@
% \acknowledgements{This work was supported by the Project Name one (XX, XXX), the Project Name two (XXXXXXXXXXXX), the Project Name three (XXXXXXXXX), and the Project Name four (XXXXXXXX). Acknowledgments are a crucial element of any published piece of work, be it professional, fictional, non-fictional, or academic. The acknowledgment section is dedicated to thanking the people that helped the author put together the writing. For example, in a thesis, the acknowledgment page will recognize the efforts of lecturers, lab assistants, and librarians that helped the writer re-search their findings.}

\appendix

% \section{Notation}\label{app:notation}

\section{Auxiliary Results for the System Model and Problem Formulation}
\label{app:system_model_aux}

\begin{proof}[proof of Proposition \ref{prop:latent_decomp}]
By the law of total probability,
\[
p(\mathbf y^{*}\mid \mathbf X^{\mathrm{loc}},\mathcal S)
=
\int
p(\mathbf y^{*}\mid \mathbf X^{\mathrm{loc}},\mathcal S,\xi)\,
p(\xi\mid \mathbf X^{\mathrm{loc}},\mathcal S)\,d\xi.
\]
Using
\(
\mathbf y^{*}\perp\!\!\!\perp \mathcal S
\mid
(\mathbf X^{\mathrm{loc}},\xi)
\),
we obtain
\[
p(\mathbf y^{*}\mid \mathbf X^{\mathrm{loc}},\mathcal S)
=
\int
p(\mathbf y^{*}\mid \mathbf X^{\mathrm{loc}},\xi)\,
p(\xi\mid \mathbf X^{\mathrm{loc}},\mathcal S)\,d\xi.
\]
\end{proof}

\begin{lemma}
% [Convexity of the action-wise fusion step]
\label{lem:convex_fusion}
For any fixed action \(a\in\mathcal{A}\), the problem
\begin{equation}
\min_{\mathbf{w}\in\Delta_{|\mathcal{M}_{a}|}}
\ell_{i,t}^{(a)}(\mathbf{w})
\end{equation}
is convex for MAE, weighted MAE, Huber loss, and squared loss on \(\mathcal{Y}\).
\end{lemma}

\begin{proof}
The fused output in Eq.~\eqref{eq:fused_output} is an affine mapping of \(\mathbf{w}\). For the loss families listed above, the map from the prediction argument to the realized loss is convex on \(\mathcal{Y}\). The composition of a convex function with an affine mapping is therefore convex in \(\mathbf{w}\).
\end{proof}

\begin{lemma}[Queue stability implies constraint satisfaction]
\label{lem:queue_stability}
If the virtual queues \(Q_{\tau}(t)\), \(Q_{c}(t)\), and \(Q_{\rho}(t)\) are mean-rate stable, i.e.,
\begin{equation}
\lim_{T\to\infty}
\frac{\mathbb{E}[Q_{g}(T)]}{T}=0,
\qquad
g\in\{\tau,c,\rho\},
\end{equation}
then the corresponding time-average constraints in Eq.~\eqref{prob:P1} are satisfied.
\end{lemma}

\begin{proof}
From Eq.~\eqref{eq:q_tau},
\begin{equation}
Q_{\tau}(t+1)
\ge
Q_{\tau}(t)
+
\frac{1}{N}\sum_{i=1}^{N}\tau_{i}^{(a_{i,t})}\big(\rho(t)\big)
-
\tau_{\max}.
\end{equation}
Summing the above inequality over \(t=0,\ldots,T-1\), taking expectations, and dividing by \(T\) gives
\begin{equation}
\frac{1}{T}\sum_{t=0}^{T-1}
\mathbb{E}\!\left[
\frac{1}{N}\sum_{i=1}^{N}\tau_{i}^{(a_{i,t})}\big(\rho(t)\big)
\right]
\le
\tau_{\max}
+
\frac{\mathbb{E}[Q_{\tau}(T)]}{T}.
\end{equation}
Taking \(\limsup_{T\to\infty}\) and using mean-rate stability yields \(\bar{\tau}\le\tau_{\max}\). The same argument applied to Eq.~\eqref{eq:q_c} and Eq.~\eqref{eq:q_rho} gives \(\bar{c}\le c_{\max}\) and \(\bar{\rho}\le \rho_{\max}\).
\end{proof}

\section{Proofs for Routing, Fusion, and Long-Term Guarantees}
\label{app:proofs_method}

\subsection{Proof of the Routing Threshold Theorem}
\label{app:proof_routing}

\begin{proof}
Fix node \(i\), slot \(t\), routing score \(r\), and cloud-load estimate \(\rho\in[0,1]\). Define the conditional surrogate drift-plus-penalty objective
\begin{align}
\bar{\Psi}_{i,t}(a;r,\rho)
&=
V\,
\mathbb{E}\!\left[
\widehat{\ell}_{i,t}^{(a)}
\,\middle|\,
r_{i,t}=r
\right]
+
Q_{\tau}(t)\tau_i^{(a)}(\rho)
\\
&\quad 
+
Q_c(t)c_i^{(a)}
+
Q_{\rho}(t)\mathbb{I}\{a=2\},
\quad a\in\{0,1,2\},
\label{eq:app_psi}
\end{align}
with the convention \(\tau_i^{(0)}(\rho)=\tau_i^{(0)}\) and \(\tau_i^{(1)}(\rho)=\tau_i^{(1)}\).
Subtracting the common baseline
\(
V\mathbb{E}[\widehat{\ell}_{i,t}^{(0)}\mid r_{i,t}=r]
+
Q_{\tau}(t)\tau_i^{(0)}
\)
does not change the minimizer. Since \(V>0\), minimizing \(\bar{\Psi}_{i,t}(a;r,\rho)\) is equivalent to minimizing \(\{J_{i,t}^{(0)},J_{i,t}^{(1)}(r),J_{i,t}^{(2)}(r,\rho)\}\).
We now compute the pairwise differences.

\textbf{Action \(1\) versus action \(0\).}
By Eq.~\eqref{eq:G1_main},
\begin{align}
\bar{\Psi}_{i,t}(1;r,\rho)-\bar{\Psi}_{i,t}(0;r,\rho)
&=
Q_{\tau}(t)\big(\tau_i^{(1)}-\tau_i^{(0)}\big)
-
V G_{1,i}(r)
\nonumber\\
&=
V J_{i,t}^{(1)}(r).
\label{eq:app_diff10}
\end{align}
Define the generalized threshold
\begin{equation}
\theta_{01,i}(t)
=
\inf
\left\{
r\in[0,1]:
G_{1,i}(r)\ge \frac{\kappa_{i,t}^{(1)}}{V}
\right\}.
\label{eq:app_theta01}
\end{equation}
Because \(G_{1,i}(r)\) is continuous and nondecreasing, we have
\[
J_{i,t}^{(1)}(r)\ge 0 \iff r<\theta_{01,i}(t),
\qquad
J_{i,t}^{(1)}(r)\le 0 \iff r\ge \theta_{01,i}(t).
\]

\textbf{Action \(2\) versus action \(0\).}
By Eq.~\eqref{eq:G1_main}--\eqref{eq:G2_main},
\begin{align}
&\bar{\Psi}_{i,t}(2;r,\rho)-\bar{\Psi}_{i,t}(0;r,\rho)
\nonumber\\
=&
Q_{\tau}(t)\big(\tau_i^{(2)}(\rho)-\tau_i^{(0)}\big)
+
Q_c(t)\kappa_i
+
Q_{\rho}(t)
-
V\bigl(G_{1,i}(r)+G_{2,i}(r)\bigr)
\nonumber\\
=&
V J_{i,t}^{(2)}(r,\rho).
\label{eq:app_diff20}
\end{align}
Define
\begin{equation}
\theta_{02,i}(t,\rho)
=
\inf
\left\{
r\in[0,1]:
G_{1,i}(r)+G_{2,i}(r)
\ge
\frac{\kappa_{i,t}^{(1)}+\kappa_{i,t}^{(2)}(\rho)}{V}
\right\}.
\label{eq:app_theta02}
\end{equation}
Again, continuity and monotonicity imply
\begin{align}
J_{i,t}^{(2)}(r,\rho)\ge 0 \iff r<\theta_{02,i}(t,\rho),
\\
J_{i,t}^{(2)}(r,\rho)\le 0 \iff r\ge \theta_{02,i}(t,\rho).
\end{align}

\textbf{Action \(2\) versus action \(1\).}
Subtracting Eq.~\eqref{eq:app_diff10} from Eq.~\eqref{eq:app_diff20} yields
\begin{align}
&\bar{\Psi}_{i,t}(2;r,\rho)-\bar{\Psi}_{i,t}(1;r,\rho)
\nonumber\\
=&
Q_{\tau}(t)\big(\tau_i^{(2)}(\rho)-\tau_i^{(1)}\big)
+
Q_c(t)\kappa_i
+
Q_{\rho}(t)
-
V G_{2,i}(r)
\nonumber\\
=&
V\left(
\frac{\kappa_{i,t}^{(2)}(\rho)}{V}-G_{2,i}(r)
\right).
\label{eq:app_diff21}
\end{align}
Define
\begin{equation}
\theta_{12,i}(t,\rho)
=
\inf
\left\{
r\in[0,1]:
G_{2,i}(r)\ge \frac{\kappa_{i,t}^{(2)}(\rho)}{V}
\right\}.
\label{eq:app_theta12}
\end{equation}
Then
\[
J_{i,t}^{(2)}(r,\rho)\le J_{i,t}^{(1)}(r)
\iff
r\ge \theta_{12,i}(t,\rho).
\]

We can now characterize the minimizer.

\textbf{Action \(0\).}
Action \(0\) is optimal if and only if
\[
J_{i,t}^{(1)}(r)\ge 0
\quad\text{and}\quad
J_{i,t}^{(2)}(r,\rho)\ge 0,
\]
which is equivalent to
\[
r<\theta_{01,i}(t)
\quad\text{and}\quad
r<\theta_{02,i}(t,\rho).
\]

\textbf{Action \(1\).}
Action \(1\) is optimal if and only if
\[
J_{i,t}^{(1)}(r)\le 0
\quad\text{and}\quad
J_{i,t}^{(1)}(r)\le J_{i,t}^{(2)}(r,\rho),
\]
which is equivalent to
\[
r\ge \theta_{01,i}(t)
\quad\text{and}\quad
r<\theta_{12,i}(t,\rho).
\]

\textbf{Action \(2\).}
Action \(2\) is optimal if and only if
\[
J_{i,t}^{(2)}(r,\rho)\le 0
\quad\text{and}\quad
J_{i,t}^{(2)}(r,\rho)\le J_{i,t}^{(1)}(r),
\]
which is equivalent to
\[
r\ge \theta_{02,i}(t,\rho)
\quad\text{and}\quad
r\ge \theta_{12,i}(t,\rho).
\]
Consequently, the minimizer is \(a_{i,t}=\arg\min_{a\in\{0,1,2\}}J_{i,t}^{(a)}\), and the induced routing policy is threshold-based in the routing score. Some regions may be empty, depending on the queue state and the node-specific costs.
Finally, the cloud-assisted branch is optimal if and only if
\[
r\ge \max\{\theta_{02,i}(t,\rho),\theta_{12,i}(t,\rho)\}.
\]
Therefore, we obtain
\begin{equation}
\theta_{c,i}(t,\rho)
=
\max\{\theta_{02,i}(t,\rho),\theta_{12,i}(t,\rho)\},
\end{equation}
which is consistent with Eq.~\eqref{eq:theta_c_main}. This completes the proof.
\end{proof}

\subsection{Proof of the Mean-Field Cloud-Load Equilibrium Proposition}
\label{app:proof_mfe}

\begin{proof}
Define the map
\begin{equation}
\mathcal{T}_t(\rho)
=
\frac{1}{N}
\sum_{i=1}^{N}
\left[
1-F_{i,t}\!\left(\theta_{c,i}(t,\rho)\right)
\right],
\qquad \rho\in[0,1].
\label{eq:app_Tmap}
\end{equation}
Since \(F_{i,t}(\cdot)\) and \(\theta_{c,i}(t,\rho)\) are continuous for every \(i\), the map \(\mathcal{T}_t(\rho)\) is continuous on \([0,1]\). Moreover, each term in the sum belongs to \([0,1]\), hence
\[
0\le \mathcal{T}_t(\rho)\le 1,
\qquad \forall \rho\in[0,1].
\]
Thus, \(\mathcal{T}_t\) maps the compact convex set \([0,1]\) into itself. By Brouwer's fixed-point theorem, there exists at least one \(\rho^{*}(t)\in[0,1]\) such that
\[
\rho^{*}(t)=\mathcal{T}_t(\rho^{*}(t)).
\]
This proves existence.
To show uniqueness, suppose in addition that \(F_{i,t}\) is differentiable with density \(f_{i,t}\), and \(\theta_{c,i}(t,\rho)\) is differentiable in \(\rho\). Let \(\rho_1,\rho_2\in[0,1]\). Then
\begin{align}
\left|\mathcal{T}_t(\rho_1)-\mathcal{T}_t(\rho_2)\right|
&\le
\frac{1}{N}
\sum_{i=1}^{N}
\left|
F_{i,t}\!\left(\theta_{c,i}(t,\rho_2)\right)
-
F_{i,t}\!\left(\theta_{c,i}(t,\rho_1)\right)
\right|.
\label{eq:app_mfe_step1}
\end{align}
By the mean-value theorem, for each \(i\) there exists a point \(\xi_i\) between \(\theta_{c,i}(t,\rho_1)\) and \(\theta_{c,i}(t,\rho_2)\) such that
\[
\left|
F_{i,t}\!\left(\theta_{c,i}(t,\rho_2)\right)
-
F_{i,t}\!\left(\theta_{c,i}(t,\rho_1)\right)
\right|
=
f_{i,t}(\xi_i)\,
\left|
\theta_{c,i}(t,\rho_2)-\theta_{c,i}(t,\rho_1)
\right|.
\]
Applying the mean-value theorem again to \(\theta_{c,i}(t,\rho)\), there exists \(\zeta_i\) between \(\rho_1\) and \(\rho_2\) such that
\[
\left|
\theta_{c,i}(t,\rho_2)-\theta_{c,i}(t,\rho_1)
\right|
=
\left|
\frac{\partial \theta_{c,i}(t,\zeta_i)}{\partial \rho}
\right|
|\rho_2-\rho_1|.
\]
Therefore,
\begin{align}
&\left|\mathcal{T}_t(\rho_1)-\mathcal{T}_t(\rho_2)\right|\nonumber\\
\le&
\frac{1}{N}
\sum_{i=1}^{N}
\left(
\sup_{r\in[0,1]}f_{i,t}(r)
\right)
\left(
\sup_{\rho\in[0,1]}
\left|
\frac{\partial \theta_{c,i}(t,\rho)}{\partial \rho}
\right|
\right)
|\rho_1-\rho_2|.
\label{eq:app_mfe_lipschitz}
\end{align}
Thus, if
\begin{equation}
\frac{1}{N}
\sum_{i=1}^{N}
\left(
\sup_{r\in[0,1]}f_{i,t}(r)
\right)
\left(
\sup_{\rho\in[0,1]}
\left|
\frac{\partial \theta_{c,i}(t,\rho)}{\partial \rho}
\right|
\right)
<1,
\label{eq:app_mfe_contraction}
\end{equation}
then \(\mathcal{T}_t\) is a contraction on \([0,1]\). By Banach's contraction theorem, the fixed point is unique.
\end{proof}

\subsection{Proof of the Retrieval-Quality Loss-Gap Proposition}
\label{app:proof_retrieval_gap}

\begin{proof}
\blue{For notational brevity, write \(\mathbf{X}=\mathbf{X}_{i,t}^{\mathrm{loc}}\), \(\mathbf{y}^{*}=\mathbf{y}_{i,t}^{*}\), \(\mathbf{z}=\mathbf{z}_{i,t}\), and \(\mathbf{z}^{\star}=\mathbf{z}_{i,t}^{\star}\). By the \(L_{\ell}\)-Lipschitz continuity of the loss in its prediction argument,}
\begin{equation}
\left|
\ell\!\left(
\mathbf{y}^{*},
g_{\vartheta}^{c}(\mathbf{X},\mathbf{z})
\right)
-
\ell\!\left(
\mathbf{y}^{*},
g^{c,\star}(\mathbf{X},\mathbf{z}^{\star})
\right)
\right|
\le
L_{\ell}
\left\|
g_{\vartheta}^{c}(\mathbf{X},\mathbf{z})
-
g^{c,\star}(\mathbf{X},\mathbf{z}^{\star})
\right\|.
\label{eq:app_retgap_1}
\end{equation}
\blue{Applying the triangle inequality gives}
\begin{align}
&\left\|
g_{\vartheta}^{c}(\mathbf{X},\mathbf{z})
-
g^{c,\star}(\mathbf{X},\mathbf{z}^{\star})
\right\|\\
\le&
\left\|
g_{\vartheta}^{c}(\mathbf{X},\mathbf{z})
-
g^{c,\star}(\mathbf{X},\mathbf{z})
\right\|
+
\left\|
g^{c,\star}(\mathbf{X},\mathbf{z})
-
g^{c,\star}(\mathbf{X},\mathbf{z}^{\star})
\right\|.
\label{eq:app_retgap_2}
\end{align}
\blue{The first term is upper bounded by \(\varepsilon_{\mathrm{pred}}\) from Eq.~\eqref{eq:pred_approx_main}, and the second term is upper bounded by \(L_z\|\mathbf{z}-\mathbf{z}^{\star}\|\) by the \(L_z\)-Lipschitz continuity of \(g^{c,\star}\) in \(\mathbf{z}\). Thus, we have}
\begin{equation}
\left\|
g_{\vartheta}^{c}(\mathbf{X},\mathbf{z})
-
g^{c,\star}(\mathbf{X},\mathbf{z}^{\star})
\right\|
\le
\varepsilon_{\mathrm{pred}}
+
L_z\|\mathbf{z}-\mathbf{z}^{\star}\|.
\label{eq:app_retgap_3}
\end{equation}
\blue{Combining Eq.~\eqref{eq:app_retgap_1} and Eq.~\eqref{eq:app_retgap_3} proves Eq.~\eqref{eq:retrieval_gap_bound_main}.}
\end{proof}

\subsection{Proof of the Inexactness Decomposition Proposition}
\label{app:proof_delta_decomp}

\begin{proof}
\blue{Let \(\Delta_{i,t}^{\mathrm{cld}}\) denote the additional loss gap induced by the factorized cloud-assisted implementation relative to the oracle retrieval-assisted predictor, i.e.,}
\begin{equation}
\Delta_{i,t}^{\mathrm{cld}}
=
\left|
\ell\!\left(
\mathbf{y}_{i,t}^{*},
g_{\vartheta}^{c}\!\left(
\mathbf{X}_{i,t}^{\mathrm{loc}},
\mathbf{z}_{i,t}
\right)
\right)
-
\ell\!\left(
\mathbf{y}_{i,t}^{*},
g^{c,\star}\!\left(
\mathbf{X}_{i,t}^{\mathrm{loc}},
\mathbf{z}_{i,t}^{\star}
\right)
\right)
\right|.
\label{eq:app_delta_cld_def}
\end{equation}
\blue{When \(a_{i,t}\neq 2\), this contribution is zero. When \(a_{i,t}=2\), Proposition~\ref{prop:retrieval_gap_main} and Eq.~\eqref{eq:ret_uniform_main} imply}
\begin{equation}
\Delta_{i,t}^{\mathrm{cld}}
\le
L_{\ell}
\left(
\varepsilon_{\mathrm{pred}}
+
L_z\varepsilon_{\mathrm{ret}}
\right).
\label{eq:app_delta_cloud}
\end{equation}
\blue{Since the drift-plus-penalty objective multiplies the one-slot average loss by \(V\), the maximal additive contribution of this architectural inexactness to the right-hand side of the one-slot drift-plus-penalty bound is}
\begin{equation}
V L_{\ell}
\left(
\varepsilon_{\mathrm{pred}}
+
L_z\varepsilon_{\mathrm{ret}}
\right).
\label{eq:app_delta_scaled}
\end{equation}
\blue{Adding the surrogate-calibration contribution \(\delta_{\mathrm{sur}}\) and the finite fixed-point iteration contribution \(\delta_{\mathrm{fp}}\) yields Eq.~\eqref{eq:delta_decomp_main}.}
\end{proof}

\subsection{Proof of the Closed-Form Fusion Update}
\label{app:proof_closed_form}

\begin{proof}
Fix a mode \(a\in\{1,2\}\). Consider the optimization problem
\begin{equation}
\min_{\mathbf{w}\in\Delta_{|\mathcal{M}_a|}}
\left\{
\eta\langle \boldsymbol{\Gamma}_{i,t-1}^{(a)},\mathbf{w}\rangle
+
D_{\mathrm{KL}}(\mathbf{w}\,\|\,\bar{\boldsymbol{\pi}}_i^{(a)})
\right\}.
\label{eq:app_ftrl}
\end{equation}
Since \(\bar{\pi}_{i,m}^{(a)}>0\) for every \(m\in\mathcal{M}_a\), the KL term is well-defined on the simplex, and the objective is strictly convex. Consequently, the optimizer is unique.
The Lagrangian is
\begin{align}
\mathcal{L}(\mathbf{w},\lambda)
=&
\eta\sum_{m\in\mathcal{M}_a}\Gamma_{i,m}^{(a)}(t-1)w_m\\
+&
\sum_{m\in\mathcal{M}_a}
w_m\log\frac{w_m}{\bar{\pi}_{i,m}^{(a)}}
+
\lambda\left(\sum_{m\in\mathcal{M}_a}w_m-1\right).\nonumber
\end{align}
Taking the derivative with respect to \(w_m\) and setting it to zero yields
\begin{equation}
\eta\Gamma_{i,m}^{(a)}(t-1)
+
\log\frac{w_m}{\bar{\pi}_{i,m}^{(a)}}
+
1+\lambda
=
0.
\end{equation}
Thus
\begin{equation}
w_m
=
\bar{\pi}_{i,m}^{(a)}
\exp\!\left(-1-\lambda-\eta\Gamma_{i,m}^{(a)}(t-1)\right).
\end{equation}
Using the simplex constraint \(\sum_m w_m=1\), we obtain
\begin{equation}
\exp(-1-\lambda)
=
\left[
\sum_{n\in\mathcal{M}_a}
\bar{\pi}_{i,n}^{(a)}
\exp\!\left(-\eta\Gamma_{i,n}^{(a)}(t-1)\right)
\right]^{-1}.
\end{equation}
Substituting this back gives
\begin{equation}
w_{i,m}^{(a)}(t)
=
\frac{
\bar{\pi}_{i,m}^{(a)}
\exp\!\left(-\eta\Gamma_{i,m}^{(a)}(t-1)\right)
}{
\sum_{n\in\mathcal{M}_a}
\bar{\pi}_{i,n}^{(a)}
\exp\!\left(-\eta\Gamma_{i,n}^{(a)}(t-1)\right)
},
\end{equation}
which is exactly Eq.~\eqref{eq:omd_closed_form_main}.
\end{proof}

\subsection{Proof of the Fusion Regret Theorem}
\label{app:proof_regret}

\begin{proof}
Fix a node \(i\) and a mode \(a\in\{1,2\}\). Let
\[
\widetilde{g}_k(\mathbf{w}),\qquad k=1,\ldots,R,
\]
be the sequence of revealed mode-\(a\) losses processed by the online updater, indexed in chronological order of label arrival. Let
\[
\widetilde{\mathbf{w}}_k
=
\arg\min_{\mathbf{w}\in\Delta_{|\mathcal{M}_a|}}
\left\{
\eta
\left\langle
\sum_{j=1}^{k-1}\widetilde{\mathbf{g}}_j,\mathbf{w}
\right\rangle
+
D_{\mathrm{KL}}\!\left(
\mathbf{w}\,\|\,\bar{\boldsymbol{\pi}}_i^{(a)}
\right)
\right\},
\]
where \(\widetilde{\mathbf{g}}_k\in\partial \widetilde{g}_k(\widetilde{\mathbf{w}}_k)\). This is exactly the entropic FTRL process induced by Eq.~\eqref{eq:ftrl_main} on the revealed-loss sequence.
For notational simplicity, abbreviate
\[
\bar{\boldsymbol{\pi}}
\triangleq
\bar{\boldsymbol{\pi}}_i^{(a)},
\qquad
\widetilde{\mathbf{g}}_k
=
(\widetilde{g}_{k,m})_{m\in\mathcal{M}_a}.
\]
By convexity of \(\widetilde{g}_k\), for any comparator \(\mathbf{u}\in\Delta_{|\mathcal{M}_a|}\),
\begin{equation}
\widetilde{g}_k(\widetilde{\mathbf{w}}_k)-\widetilde{g}_k(\mathbf{u})
\le
\langle \widetilde{\mathbf{g}}_k,\widetilde{\mathbf{w}}_k-\mathbf{u}\rangle.
\label{eq:app_convex_step}
\end{equation}
Thus, it is sufficient to bound the linearized regret
\begin{equation}
\sum_{k=1}^{R}
\langle \widetilde{\mathbf{g}}_k,\widetilde{\mathbf{w}}_k-\mathbf{u}\rangle.
\label{eq:app_lin_regret}
\end{equation}

Define the potential
\begin{equation}
W_k
=
\sum_{m\in\mathcal{M}_a}
\bar{\pi}_m
\exp\!\left(
-\eta\sum_{j=1}^{k-1}\widetilde{g}_{j,m}
\right).
\label{eq:app_weight_sum}
\end{equation}
By Lemma~\ref{lem:closed_form_main},
\begin{equation}
\widetilde{w}_{k,m}
=
\frac{
\bar{\pi}_m
\exp\!\left(
-\eta\sum_{j=1}^{k-1}\widetilde{g}_{j,m}
\right)
}{W_k}.
\label{eq:app_weight_def}
\end{equation}
Therefore
\begin{equation}
\frac{W_{k+1}}{W_k}
=
\sum_{m\in\mathcal{M}_a}
\widetilde{w}_{k,m}\exp(-\eta \widetilde{g}_{k,m}).
\end{equation}
Using Hoeffding's lemma and the boundedness
\(
\|\widetilde{\mathbf{g}}_k\|_{\infty}\le G_a
\),
we obtain
\begin{equation}
\log\frac{W_{k+1}}{W_k}
\le
-\eta\langle \widetilde{\mathbf{w}}_k,\widetilde{\mathbf{g}}_k\rangle
+
\frac{\eta^2 G_a^2}{2}.
\label{eq:app_upper_W}
\end{equation}
Summing over \(k=1,\ldots,R\) yields
\begin{equation}
\log W_{R+1}-\log W_1
\le
-\eta
\sum_{k=1}^{R}
\langle \widetilde{\mathbf{w}}_k,\widetilde{\mathbf{g}}_k\rangle
+
\frac{\eta^2 G_a^2 R}{2}.
\label{eq:app_sum_upper}
\end{equation}

On the other hand, for any comparator \(\mathbf{u}\in\Delta_{|\mathcal{M}_a|}\),
\begin{align}
W_{R+1}
&=
\sum_m \bar{\pi}_m
\exp\!\left(
-\eta\sum_{k=1}^{R}\widetilde{g}_{k,m}
\right)
\nonumber\\
&\ge
\prod_m
\left[
\frac{\bar{\pi}_m}{u_m}
\exp\!\left(
-\eta\sum_{k=1}^{R}\widetilde{g}_{k,m}
\right)
\right]^{u_m}
\end{align}
by the weighted AM-GM inequality. Taking logarithms gives
\begin{equation}
\log W_{R+1}
\ge
-\eta\sum_{k=1}^{R}\langle \mathbf{u},\widetilde{\mathbf{g}}_k\rangle
-
D_{\mathrm{KL}}(\mathbf{u}\,\|\,\bar{\boldsymbol{\pi}}).
\label{eq:app_lower_W}
\end{equation}
Since \(W_1=\sum_m\bar{\pi}_m=1\), we have \(\log W_1=0\). Combining Eq.~\eqref{eq:app_sum_upper} and Eq.~\eqref{eq:app_lower_W}, we obtain
\begin{equation}
\sum_{k=1}^{R}
\langle \widetilde{\mathbf{w}}_k-\mathbf{u},\widetilde{\mathbf{g}}_k\rangle
\le
\frac{
D_{\mathrm{KL}}(\mathbf{u}\,\|\,\bar{\boldsymbol{\pi}})
}{\eta}
+
\frac{\eta G_a^2 R}{2}.
\label{eq:app_lin_bound}
\end{equation}
Finally, combining Eq.~\eqref{eq:app_convex_step} and Eq.~\eqref{eq:app_lin_bound} yields
\begin{equation}
\sum_{k=1}^{R}
\widetilde{g}_k(\widetilde{\mathbf{w}}_k)-\widetilde{g}_k(\mathbf{u})
\le
\frac{
D_{\mathrm{KL}}(\mathbf{u}\,\|\,\bar{\boldsymbol{\pi}})
}{\eta}
+
\frac{\eta G_a^2 R}{2}.
\end{equation}
Choosing
\[
\eta
=
\sqrt{
\frac{
2D_{\mathrm{KL}}(\mathbf{u}\,\|\,\bar{\boldsymbol{\pi}})
}{
G_a^2 R
}
}
\]
gives
\[
\mathrm{Regret}_{i,a}(R)
=
O(\sqrt{R}).
\]
This completes the proof.
\end{proof}

\subsection{Proof of the Long-Term Performance Guarantee}
\label{app:proof_lyapunov}

\begin{proof}
Let
\begin{equation}
Y_{\tau}(t)
=
\frac{1}{N}\sum_{i=1}^{N}\tau_i^{(a_{i,t})}(\rho(t)),
~
Y_c(t)
=
\frac{1}{N}\sum_{i=1}^{N}c_i^{(a_{i,t})},
~
Y_{\rho}(t)=\rho(t),\nonumber
\end{equation}
and define the Lyapunov function
\begin{equation}
\mathcal{L}(\mathbf{Q}(t))
=
\frac{1}{2}
\left(
Q_{\tau}^2(t)+Q_c^2(t)+Q_{\rho}^2(t)
\right),
\label{eq:app_lyap}
\end{equation}
where \(\mathbf{Q}(t)\triangleq(Q_{\tau}(t),Q_c(t),Q_{\rho}(t))\). The conditional drift is
\[
\Delta(t)
=
\mathbb{E}\!\left[
\mathcal{L}(\mathbf{Q}(t+1))-\mathcal{L}(\mathbf{Q}(t))
\,\middle|\,
\mathbf{Q}(t)
\right].
\]
From the queue updates Eq.~\eqref{eq:q_tau}--\eqref{eq:q_rho} and the inequality \(([x]^+)^2\le x^2\), we obtain
\begin{align}
&\mathcal{L}(\mathbf{Q}(t+1))-\mathcal{L}(\mathbf{Q}(t))
\\
&\le
\frac{1}{2}
\Big(
(Y_{\tau}(t)-\tau_{\max})^2
+
(Y_c(t)-c_{\max})^2
+
(Y_{\rho}(t)-\rho_{\max})^2
\Big)
\nonumber\\
&
+
Q_{\tau}(t)(Y_{\tau}(t)-\tau_{\max})
+
Q_c(t)(Y_c(t)-c_{\max})
\nonumber\\
&+
Q_{\rho}(t)(Y_{\rho}(t)-\rho_{\max}).\nonumber
\label{eq:app_drift_expand}
\end{align}
Because all one-slot quantities are bounded, there exists a finite constant \(B_0\) such that
\begin{equation}
\frac{1}{2}
\Big(
(Y_{\tau}(t)-\tau_{\max})^2
+
(Y_c(t)-c_{\max})^2
+
(Y_{\rho}(t)-\rho_{\max})^2
\Big)
\le B_0.
\label{eq:app_B0}
\end{equation}
Taking the conditional expectation yields
\begin{align}
\Delta(t)
&\le
B_0
+
Q_{\tau}(t)\mathbb{E}[Y_{\tau}(t)-\tau_{\max}\mid \mathbf{Q}(t)]
\\
&
+
Q_c(t)\mathbb{E}[Y_c(t)-c_{\max}\mid \mathbf{Q}(t)]
+
Q_{\rho}(t)\mathbb{E}[Y_{\rho}(t)-\rho_{\max}\mid \mathbf{Q}(t)].\nonumber
\label{eq:app_drift_bd}
\end{align}
Adding \(V\mathbb{E}[L(t)\mid \mathbf{Q}(t)]\) to both sides gives
\begin{align}
&\Delta(t)+V\mathbb{E}[L(t)\mid \mathbf{Q}(t)]
\\
&\le
B_0
+
\mathbb{E}
\Big[
V L(t)
+
Q_{\tau}(t)(Y_{\tau}(t)-\tau_{\max})
\nonumber\\
&\quad
+
Q_c(t)(Y_c(t)-c_{\max})
 +
Q_{\rho}(t)(Y_{\rho}(t)-\rho_{\max})
\;\Big|\;\mathbf{Q}(t)
\Big].
\label{eq:app_dpp}
\end{align}
According to the theorem statement, the implemented controller minimizes the right-hand side of Eq.~\eqref{eq:app_dpp} within additive error \(\delta\). \blue{By Proposition~\ref{prop:delta_decomp_main}, one admissible concrete choice is \(\delta=\delta_{\mathrm{sur}}+\delta_{\mathrm{fp}}+V L_{\ell}(\varepsilon_{\mathrm{pred}}+L_z\varepsilon_{\mathrm{ret}})\); however, the drift analysis below only requires a finite uniform additive bound.}

\textbf{Part 1: optimality gap.}
Let \(\Pi^{*}\) be an optimal stationary policy with long-term average loss \(L^{*}\) and feasible average constraints. Plugging \(\Pi^{*}\) into the right-hand side of Eq.~\eqref{eq:app_dpp} yields
\begin{equation}
\Delta(t)+V\mathbb{E}[L(t)\mid \mathbf{Q}(t)]
\le
B_0+\delta+V L^{*}.
\label{eq:app_compare_opt}
\end{equation}
Taking total expectation and summing over \(t=0,\ldots,T-1\), the Lyapunov terms telescope:
\begin{equation}
\mathbb{E}[\mathcal{L}(\mathbf{Q}(T))]-\mathbb{E}[\mathcal{L}(\mathbf{Q}(0))]
+
V\sum_{t=0}^{T-1}\mathbb{E}[L(t)]
\le
T(B_0+\delta+V L^{*}).
\label{eq:app_telescope_opt}
\end{equation}
Since \(\mathcal{L}(\mathbf{Q}(T))\ge 0\),
\begin{equation}
V\sum_{t=0}^{T-1}\mathbb{E}[L(t)]
\le
T(B_0+\delta+V L^{*})
+
\mathbb{E}[\mathcal{L}(\mathbf{Q}(0))].
\end{equation}
Divide by \(VT\) and let \(T\to\infty\) to obtain
\begin{equation}
\bar{L}^{\mathrm{alg}}
\le
L^{*}+\frac{B_0+\delta}{V}.
\end{equation}

\textbf{Part 2: average queue backlog.}
Let \(\Pi^{\xi}\) be a Slater policy satisfying the constraints with slack \(\xi>0\). Then
\begin{align}
\mathbb{E}[Y_{\tau}^{\Pi^\xi}(t)]\le \tau_{\max}-\xi,\\
\mathbb{E}[Y_c^{\Pi^\xi}(t)]\le c_{\max}-\xi,\\
\mathbb{E}[Y_{\rho}^{\Pi^\xi}(t)]\le \rho_{\max}-\xi.
\label{eq:app_slater}
\end{align}
Substituting \(\Pi^{\xi}\) into Eq.~\eqref{eq:app_dpp} gives
\begin{align}
&\Delta(t)+V\mathbb{E}[L(t)\mid \mathbf{Q}(t)]\nonumber\\
\le&
B_0+\delta+V L^{\xi}
-
\xi\big(Q_{\tau}(t)+Q_c(t)+Q_{\rho}(t)\big),
\label{eq:app_compare_slater}
\end{align}
where \(L^{\xi}\) is the average loss of \(\Pi^{\xi}\). Dropping the nonnegative term \(V\mathbb{E}[L(t)\mid \mathbf{Q}(t)]\) on the left-hand side and summing over \(t=0,\ldots,T-1\) yields
\begin{align}
&\mathbb{E}[\mathcal{L}(\mathbf{Q}(T))]-\mathbb{E}[\mathcal{L}(\mathbf{Q}(0))]\\
\le &
T(B_0+\delta+V L^{\xi})
-
\xi\sum_{t=0}^{T-1}
\mathbb{E}[Q_{\tau}(t)+Q_c(t)+Q_{\rho}(t)].\nonumber
\label{eq:app_telescope_slater}
\end{align}
Rearranging,
\begin{equation}
\xi\sum_{t=0}^{T-1}
\mathbb{E}[Q_{\tau}(t)+Q_c(t)+Q_{\rho}(t)]
\le
T(B_0+\delta+V L^{\xi})
+
\mathbb{E}[\mathcal{L}(\mathbf{Q}(0))].
\end{equation}
Divide by \(T\xi\) and let \(T\to\infty\):
\begin{equation}
\limsup_{T\to\infty}
\frac{1}{T}
\sum_{t=0}^{T-1}
\mathbb{E}[Q_{\tau}(t)+Q_c(t)+Q_{\rho}(t)]
\le
\frac{B_0+\delta+V L^{\xi}}{\xi}.
\label{eq:app_avg_queue_bd}
\end{equation}
Hence, the average queue backlog is \(O(V)\).

\textbf{Part 3: mean-rate stability.}
Assume that the one-slot queue increments are uniformly bounded, i.e., for each \(g\in\{\tau,c,\rho\}\) there exists \(C_g<\infty\) such that
\[
|Q_g(t+1)-Q_g(t)|\le C_g,\qquad \forall t.
\]
Fix any \(\alpha\in(0,1)\). For any \(t\in\{\lfloor \alpha T\rfloor,\ldots,T-1\}\),
\[
Q_g(T)\le Q_g(t)+C_g(T-t)\le Q_g(t)+C_g(1-\alpha)T.
\]
Averaging over \(t=\lfloor \alpha T\rfloor,\ldots,T-1\), taking expectation, and dividing by \(T\), we obtain
\begin{equation}
\frac{\mathbb{E}[Q_g(T)]}{T}
\le
\frac{1}{(1-\alpha)T^2}
\sum_{t=\lfloor \alpha T\rfloor}^{T-1}\mathbb{E}[Q_g(t)]
+
C_g(1-\alpha).
\label{eq:app_mrs}
\end{equation}
Because Eq.~\eqref{eq:app_avg_queue_bd} implies finite time-average backlog, the first term on the right-hand side of Eq.~\eqref{eq:app_mrs} vanishes as \(T\to\infty\). Therefore
\[
\limsup_{T\to\infty}\frac{\mathbb{E}[Q_g(T)]}{T}
\le
C_g(1-\alpha).
\]
Letting \(\alpha\uparrow 1\) yields
\[
\lim_{T\to\infty}\frac{\mathbb{E}[Q_g(T)]}{T}=0,
\qquad g\in\{\tau,c,\rho\}.
\]
Thus, all virtual queues are mean-rate stable. By Lemma~\ref{lem:queue_stability}, the time-average latency, communication, and cloud-usage constraints are satisfied. This completes the proof.
\end{proof}

% \input{Chapter/c6:conclusion}

%%@@@@@@@@@@@@@@@@@@@@@@@@@@@@@@@@@@@@@@@@@@@@
%\bibliographystyle{EMSart}
\bibliographystyle{IEEEtran}
\bibliography{refs}

\end{document}